\documentclass{article}

\PassOptionsToPackage{numbers, compress}{natbib}
\usepackage[preprint]{neurips_2026}

\usepackage[utf8]{inputenc} 
\usepackage[T1]{fontenc}    
\usepackage{hyperref}       
\usepackage{url}            
\usepackage{booktabs}       
\usepackage{amsfonts}       
\usepackage{nicefrac}       
\usepackage{microtype}      
\usepackage{xcolor}         

\usepackage{microtype}
\usepackage{graphicx}
\usepackage{subcaption}
\usepackage{booktabs}

\usepackage{amsmath}
\usepackage{amssymb}
\usepackage{mathtools}
\usepackage{amsthm}
\usepackage{thmtools}
\usepackage{amsfonts}
\usepackage{thm-restate}
\usepackage{enumitem}
\usepackage{comment}
\usepackage[table]{xcolor}
\allowdisplaybreaks

\usepackage{algorithm}
\usepackage{algorithmic}

\usepackage{thm-restate}

\usepackage[capitalize,noabbrev]{cleveref}

\usepackage{mymacros}

\declaretheorem[name=Definition,numberwithin=section]{definition}

\declaretheorem[name=Assumption,numberwithin=section]{assumption}

\declaretheorem[name=Lemma,sharenumber=theorem]{lemma}

\usepackage[textsize=tiny]{todonotes}

\title{How Log-Barrier Helps\\Exploration in Policy Optimization}

\author{%
  Leonardo Cesani\\
  Politecnico di Milano, Milan, Italy \\
  \texttt{leonardo.cesani@polimi.it} \\
  \And
  Matteo Papini \\
  Università degli Studi di Milano, Milan, Italy \\
  \texttt{matteo.papini@unimi.it} \\
  \AND
  Marcello Restelli \\
  Politecnico di Milano, Milan, Italy \\
  \texttt{marcello.restelli@polimi.it} \\
}

\begin{document}

\maketitle

\begin{abstract}
  Recently, it has been shown that the Stochastic Gradient Bandit (\sgb) algorithm converges to a globally optimal policy with a constant learning rate. However, these guarantees rely on unrealistic assumptions about the learning process, namely that the probability of playing the optimal action is always bounded away from zero. We attribute this to the lack of an explicit exploration mechanism in \sgb. To solve this issue, we investigate the information geometry of the problem to enforce exploration explicitly. This naturally leads to regularizing the \sgb objective with a log-barrier on the parametric policy, structurally encouraging a minimal level of exploration. Our formulation aligns with the Natural Policy Gradient, as both methods exploit the underlying geometry of the policy space encoded by the Fisher information.  We prove that Log-Barrier Stochastic Gradient Bandit (\lbsgb) matches the sample complexity of \sgb, but also converges (at a slower rate) without any assumption on the sampling probability of the optimal action. Finally, we validate our theoretical findings through numerical simulations, showing the benefits of the log-barrier regularization, especially when the number of actions is large.
\end{abstract}


\section{Introduction}
\label{sec:introduction}
Stochastic gradient algorithms have been widely studied in the context of Reinforcement Learning~\citep[RL,][]{sutton2018rl} and Multi-Armed Bandit~\citep[MAB,][]{lattimore2020bandit} problems. Starting from foundational algorithms such as \texttt{REINFORCE}~\cite{williams1992simple}, and \texttt{GPOMDP}/\texttt{PGT}~\cite{baxter2001infinite,sutton1999policy}, Policy Gradient (PG) methods have become a standard tool for direct policy optimization. These methods have also been effectively adapted to modern deep architectures, leading to the development of algorithms showing impressive performance \cite{lillicrap2015continuous, schulman2017proximal, Ouyang0JAWMZASR22, haarnoja2018soft, AhmadianCGFKPUH24, grpo}. 

Despite their empirical success, providing rigorous finite-time convergence guarantees for PG methods remains a significant challenge~\citep{agarwal2021theory, yuan2022general}.
\citet{mei2022global} proposed to study policy gradient methods in the simplified setting of MABs with parametric softmax policies, to gain insights on the behavior of this family of algorithms that could transfer to the more general RL problem.
In this ``gradient bandit" setting~\citep[Section 2.7][]{sutton2018rl}, \texttt{REINFORCE} assumes a very simple form and is amenable to detailed analysis. Following~\citet{mei2024bandit}, we call it Stochastic Gradient Bandit (\sgb).
In this work, we will focus on the sample complexity of gradient bandits, that is, the number of samples or pulls required to find an $\epsilon$-optimal policy.\footnote{Sample complexity results can be easily translated into convergence rates~\citep{mei2024bandit}. Also note that iteration complexity coincides with sample complexity in the bandit setting. For a discussion of regret guarantees, see~\cite{baudrydoes} instead. For detailed related works, refer to Appendix~\ref{apx:detailed_related_works}.} %
\citet{mei2024bandit} showed that \sgb enjoys $\mathcal{O}(\epsilon^{-1})$ sample complexity. However, \citet{baudrydoes} highlighted the dependence of this upper bound on a potentially large problem-dependent constant that was previously neglected and hides a problematic assumption on how often the learning algorithm samples the optimal arm. 

This could be part of a broader problem. In fact, unlike classic bandit algorithms that explicitly address the exploration-exploitation dilemma through optimism in the face of uncertainty~\cite{auer2002finite} or posterior sampling~\cite{agrawal2012analysis}, vanilla PG methods such as \sgb lack a direct mechanism to control exploration, relying solely on the policy's stochasticity. As the gradient updates drive the policy toward the boundary of the probability simplex, the gradient itself vanishes, possibly leading to premature convergence to sub-optimal policies.
To mitigate this problem, practitioners often resort to entropy regularization~\citep{haarnoja2018soft}, which, in the context of gradient bandits, has been studied only in combination with exact gradients~\cite{mei2022global}. Several works in the RL literature have argued that, although speeding up convergence by smoothing the objective landscape, entropy regularization provides insufficient support to exploration~\cite{ahmed2019understanding,bolland2025myth,MeiDXSS21,odonoghue2020makingsensereinforcementlearning}. 

Motivated by these limitations, we investigate the underlying geometry of the policy optimization landscape to establish a framework for exploration. Rather than merely regularizing performance with the policy entropy or some other regularizer, inspired by the literature on Natural Policy Gradient~\citep[\npg, ][]{kakade2001natural, mei2022role} and Fisher-non-degenerate policies \citep{liu2020improved,ding2022globaloptimumconvergencemomentumbased}, we seek a solution that preserves the well-conditioning of the Fisher information matrix. This naturally leads to the adoption of a \emph{log-barrier} regularization~\cite{agarwal2021theory, yuan2022general} to actively prevent the degeneration of the Fisher geometry near the boundaries of the probability simplex. Specifically, we show that, in the gradient bandit setting, applying a log-barrier to the spectrum of the Fisher Information Matrix is mathematically equivalent to applying a log-barrier to the parametric policy. Thanks to this equivalence, we provide strong geometric motivation for the log-barrier regularization, explicitly encouraging exploration in \sgb, and proving convergence without relying on strong implicit assumptions that may fail due to unlucky realizations.


\textbf{Contributions}. We adopt the MAB framework as a test bed to analyze the dynamics of policy optimization with log-barrier regularization. In particular:
\begin{itemize}[topsep=-4pt,itemsep=0pt,parsep=0pt,partopsep=0pt]
    \item We show a connection between \npg and log-barrier regularization, contributing to the understanding of the role of the Fisher information in exploration;
    \item We introduce Log-Barrier Stochastic Gradient Bandit (\lbsgb), a PG algorithm designed to encourage persistent exploration by keeping action probabilities in the simplex interior 
    \item We analyze the convergence guarantees of \lbsgb, establishing $\widetilde{\mathcal{O}}(\epsilon^{-1})$ sample complexity, comparable to state-of-the-art algorithms under the same assumptions;
    \item We show that \lbsgb eliminates the need for an implicit assumption regarding the sampling probability of the optimal arm during the learning process, at the cost of converging at a slower rate, $\mathcal{O}(\epsilon^{-5})$;
    \item We empirically show that log-barrier regularization improves the performance as the number of arms $K$ increases compared to (entropy-regularized) \sgb and \npg.
\end{itemize}
At the algorithmic level, \lbsgb modifies the \sgb objective by adding a log-barrier term on the policy, and updating the gradient with a deterministic penalty for small actions. The rest of the paper explains why this simple modification has a Fisher-geometric interpretation and yields stronger exploration guarantees. Complete proofs of all theoretical statements are provided in Appendix~\ref{apx:proofs}.

\section{Preliminaries}
\label{sec:preliminaries}

\textbf{Notation.}  For a measurable set $\mathcal{X}$, we denote by $\Delta(\mathcal{X})$ the set of probability measures over $\mathcal{X}$. For $P \in \Delta(\mathcal{X})$, we denote with $p$ its density function w.r.t. a reference measure that we assume to exist whenever needed. We will use $x \sim P$ to express that random variable $x$ is distributed according to $P$. For $n\in \mN$, we denote $\dsb{n} \coloneqq \{1, \dots, n\}$. For $a \in \mR$, we define $(a)^{+} = \max\{0, a\} .$ Given a symmetric matrix $A \in \mR^{d\times d}$, we denote its eigenvalues with $\mu_i $, where $i \in \dsb{d}$. Given two matrices $A$ and $B$, we say that $A \succeq B$ if the matrix $A-B$ is positive semidefinite.

\textbf{$K$-armed Bandit}. We formulate the stochastic $K$-armed bandit as a sequential decision-making problem where, at each time $t \in \dsb{T}$, the learner chooses an action $a_t \in \dsb{K}$ and receives a reward $R_t(a_t)$. This reward $R_{t}(a_t)$ is the $a_t$-th component of a random reward vector $\bm{R}_t \in \mR^K$, whose components are sampled i.i.d. from an unknown multivariate distribution $\nu$ supported on $[-\Rmax, \Rmax]^K$. 
Given $\nu_a$  as the $a$-th component of $\nu$, the mean reward of each action is defined as $r(a) \coloneqq \mE_{R_t(a)\sim\nu_a}\left[R_t(a)\right]$. The vector of mean rewards for each arm is denoted $\bm{r} \coloneqq (r(a))_{a\in \dsb{K}}$.
Following~\citet{mei2024bandit}, we make the following assumption on $\bm{r}$.

\begin{assumption}(No ties)
\label{asm:no_ties_r}
For all $a, a'\in \dsb{K}$, if $a\ne  a'$, then $r(a) \ne r(a')$.
\end{assumption}
Furthermore, we define the optimal arm as $a^* \coloneqq \argmax_{a \in \dsb{K}} r(a) $ which, under Assumption~\ref{asm:no_ties_r}, is unique and the optimal deterministic policy $\pi^* \coloneqq \argmax_{\pi} \pi^{\top}\bm{r}$. Finally, we denote the optimal reward gap as $\Delta^*\coloneqq \min_{a \ne a^*} \left|r(a^*) - r(a)\right|$, and the minimum reward gap as $\Delta = \min_{a \ne a'} \left|r(a) - r(a')\right|$. Note that $\Delta^*\ge\Delta>0$ under Assumption~\ref{asm:no_ties_r}. 

\textbf{Stochastic Gradient Bandit Algorithm}. The Stochastic Gradient Bandit (\sgb) algorithm \cite{sutton2018rl} parametrizes a stochastic policy $\pi_{\bt}$ over actions $a \in \dsb{K}$ using a softmax distribution:
\begin{align}
    \pt(a) = \frac{\exp\{\bt(a)\}}{\sum_{b \in \dsb{K}} \exp\{\bt(b)\}}, 
    \label{eq:softmax}
\end{align}
where $\bt \in \mR^K$ is a vector of parameters representing action preferences.
Starting from an initial parameter vector $\bt_1  = \bm{0}_K$, inducing a uniform policy, the algorithm seeks to learn a policy maximizing the expected reward $J(\bt) \coloneqq \pi_{\bt}^{\top} \bm{r}$, where $\bm{r} \in \mR^K$ is the vector of mean rewards. Since $\bm{r}$ is unknown, \sgb employs Stochastic Gradient Ascent (\texttt{SGA}): $\bt_{t+1} \leftarrow \bt_t + \alpha \nabla_{\bt_t} \left( \pi_{\bt_t}^{\top} \hat{\bm{r}}_t \right)$, where $\hat{\bm{r}}_t \in \mR^K$ is an unbiased importance-sampling estimator of $\bm{r}$ with components $\hat{r}_t(a) \coloneqq \frac{\mI\{a_t = a\}}{\pi_{\bt_t}(a)} R_t(a_t)$. The stochastic gradient $\nabla_{\bt} \left( \pi_{\bt_t}^{\top} \hat{\bm{r}}_t \right)$ is an unbiased estimator of $\nabla_{\bt_t} \left( \pi_{\bt_t}^{\top} \bm{r} \right) = \left(\text{diag}(\pi_{\bt_t}) - \pi_{\bt_t} \pi_{\bt_t}^{\top}\right) \bm{r}$~\cite{mei2024bandit}. 

\textbf{Natural Policy Gradient in Bandits}. The Natural Policy Gradient (\npg) algorithm~\cite{kakade2001natural} precondition the gradient update of \sgb with the inverse of the Fisher Information Matrix (FIM), which is defined as:
\begin{align*}
    \fim &= \mE_{a\sim \pt}\left[ \nabla_{\bt} \log\pi_{\bt}(a)\nabla_{\bt} \log\pi_{\bt}(a)^{\top} \right] = \text{diag}(\pi_{\bt}) - \pi_{\bt} \pi_{\bt}^{\top},
\end{align*}
The FIM acts as the Riemannian metric tensor of the probability simplex. Because the Euclidean metric between parameters does not reflect the true divergence between the associated distributions, a Euclidean step in $\bt$ can result in arbitrary changes to the actual policy $\pi_{\bt}$ depending on its proximity to the boundaries. \npg corrects this issue by updating the parameters in the direction of $\fim^{-1}\nabla_{\bt}J(\bt)$. This adjusts the optimization step according to the true local curvature of the policy space~\cite{amari1998natural}, ensuring the update follows the steepest ascent direction, and controls the actual statistical distance between successive policy distributions. In \sgb, the linearly dependent parameters of the softmax policy (Equation~\eqref{eq:softmax}) make the FIM singular, precluding explicit computation of the natural gradient. To obtain the \npg update rule, it is necessary to find the solutions of the linear system $\fim \cdot \bm{x} = \bm{g}(a_t)$, where $\bm{g}(a_t) = \nabla _{\bt}\log \pi_{\bt}(a_t)= \mathbf{e}_{a_t} - \pt$ being the score of the policy given the sampled arm $a_t$, and $\mathbf{e}_{a_t}$ denotes the one-hot encoding of action $a_t$~\citep[Appendix D.2,][]{chung2021beyond}. Solving the system yields an update rule, modifying only the parameter associated to the sampled action: $\bt_{t+1} \leftarrow \bt_{t} + \alpha \cdot R_t(a_t) \left( \mathbf{e}_{a_t}/\pi_{\bt_t}(a_t)\right)$. Details of \sgb and \npg are provided in Appendix~\ref{apx:algorithms}.

\section{Limitations of \sgb and \npg}
\label{sec:limitations}
In this section, we outline the theoretical limitations of both \sgb and \npg. We first review a structural flaw in existing convergence analyses for \sgb discovered by~\citet{baudrydoes}, who showed that these analyses rely on implicit exploration assumptions that fail in worst-case scenarios. Then, we examine \npg, reviewing its ``over-committal" behavior in the bandit setting highlighted by~\citet{chung2021beyond}.

\textbf{Limitations of \sgb.} The convergence analysis of \sgb in~\citet{mei2024bandit} relies on a hidden assumption regarding the probability of playing the optimal arm, originating from a flaw in the proof. As identified by~\citet{baudrydoes}, when this implicit assumption is violated, the resulting sample complexity result may become vacuous. Specifically, the analysis in~\citet{mei2024bandit} overlooks the impact of ``extreme" events, i.e., low-probability trajectories where an initial sequence of unfavorable rewards causes the sampling probability of the optimal arm $\pi_{\bt_t}(a^*)$ to vanish. In these scenarios, the constant governing the sample complexity result may become unbounded.  This discrepancy originates from a structural error in the convergence proof of \sgb~\citep[Theorem 5.5,][]{mei2024bandit}, in which the term $\zeta^* = \inf_{t \geq 1} \mathbb{E}[\pi^2_{\bt_{t}}(a^*)]$ is treated as a constant independent of the 
sub-optimality gap $\delta(\boldsymbol{\theta}) \coloneqq (\pi^* - \pi_{\bt})^{\top} \bm{r}$. In the stochastic setting, the sampling probability $\pi_{\bt}(a^*)$ is a trajectory-dependent random variable; therefore, its infimum cannot be assumed independent of the policy dynamics, as the history of rewards determines both the current parameters $\bt_t$ and the probability of the optimal arm. To address this analytical flaw, \citet{baudrydoes} propose shifting the dependence to the second moment of the reciprocal probability, defining $c^* \coloneqq \sup_t\mE \left[\pi_{\bt_t}(a^*)^{-2}\right]$, which yields a sample complexity of $\mathcal{O}\left(c^*\epsilon^{-1}\right)$. However, while~\citet{mei2024bandit} prove that the probability of the optimal arm remains strictly positive (Theorem 5.1), this does not preclude the policy from becoming arbitrarily close to zero on specific trajectories. Guaranteeing $c^*$ to be bounded requires a stronger condition, namely a uniform lower bound $\pi_{\bt_t}(a^*) \geq C > 0$ almost surely. Furthermore, even if a finite upper bound of $c^*$ were to exist, it remains theoretically uncharacterized in the standard \sgb framework. Without explicit control over ``extreme" events, where the policy approaches the simplex boundary~\cite{baudrydoes, fan2025fragility}, the second moment of the reciprocal probability may diverge, making the resulting convergence guarantee vacuous.

\textbf{Limitations of \npg.} While \npg benefits from a geometry-aware update direction, it is known to suffer from premature convergence to suboptimal arms due to its aggressive step sizes. Specifically,~\citet{mei2022role} proved that unregularized \npg has a strictly positive probability of converging to sub-optimal actions. This issue stems from the algorithm's tendency to drive the policy parameters toward the simplex boundaries too rapidly. As noted by~\citet{mei2022role}, this results in an ``over-committal'' behavior: the algorithm may repeatedly sample a potentially suboptimal arm, actively suppressing the exploration required to identify the true optimal action. To mitigate this effect,~\citet{mei2022role} and~\citet{chung2021beyond} demonstrate that subtracting a value baseline from the importance-sampling estimator can penalize these aggressive updates. However, while~\citet{mei2022role} provide a proof of global convergence for this baseline-corrected \npg (Theorem 1), their analysis relies on the exact same structural flaw identified for \sgb. Because the baseline only acts on the gradient's variance and does not structurally prevent the policy from approaching the simplex boundary on ``extreme'' trajectories, exploration can still vanish. Consequently, the resulting theoretical convergence bounds remain potentially vacuous in worst-case scenarios.

\textbf{Exploration and Geometry.} We argue that the failures of both \sgb and \npg stem from how they deal with the geometry of the probability simplex under the softmax parametrization. In fact, near the boundaries of the simplex, the softmax mapping becomes degenerate, and the policy becomes insensitive to parameter updates. This distortion is captured by the spectrum of the FIM, whose eigenvalues vanish as the policy approaches a deterministic vertex~\cite{ding2022globaloptimumconvergencemomentumbased}. While in \sgb this causes the policy to collapse into the boundary where gradients vanish, and exploration is extinguished (causing the theoretical divergence of $c^*$), \npg accounts for the curvature of the optimization landscape, but incurring the ``over-committal behavior due to the aggressiveness of its update rule. Because both algorithms rely on the policy's internal stochasticity to explore, a collapse toward the boundary of the probability simplex inherently extinguishes exploration. To overcome these limitations, we propose decoupling reward maximization from the exploration requirement by exploiting these information-geometric properties in different ways. By exploiting the Fisher geometry to prevent the policy from collapsing toward the simplex boundary, we can provide \sgb with an explicit exploration mechanism while entirely avoiding the over-committal behavior of \npg.

\section{How Log-barrier Helps Exploration}
\label{sec:comparison}

In this section, we highlight the fundamental role of Fisher information in driving exploration during policy optimization. We begin by examining a theoretical regularity condition that is typically taken as a prerequisite for global convergence in PG methods, and reinterpret it as an explicit exploration requirement. Building on this, we propose a method to structurally enforce this geometric condition rather than merely assuming it holds.

\textbf{Fisher-non-degeneracy}. While the aggressive behavior of \npg often leads to instability in practice, it serves as a fundamental pillar for establishing global convergence results of PG. These theoretical guarantees, however, typically necessitate a specific regularity condition on the policy's geometry to ensure the FIM remains a well-behaved metric for the update. Theoretical analyses of PG methods with generic policy parametrization often assume Fisher non-degeneracy to ensure global convergence rates~\cite{ding2022globaloptimumconvergencemomentumbased, fatkhullin2023stochastic,liu2020improved}. This condition requires well-behaved curvature of the parameter space so that the FIM is a strictly positive-definite preconditioner of the gradient of the performance index. 

\begin{definition}[Fisher-non-degeneracy]
    \label{def:fnd}
    A policy class parametrized by $\bt \in \mR^{K}$, is said to be \emph{Fisher-non-degenerate} if there exist a constant $\mu_{\text{F}}>0$ such that for all $\bt \in \mR^K$, the induced FIM satisfied $\fim \succeq \mu_{\text{F}} \bm{I}$, where $\bm{I}$ is the identity matrix.
\end{definition}

While \citet{ding2022globaloptimumconvergencemomentumbased} argues that the Fisher-non-degenerate setting implicitly guarantees that the agent effectively explores the action space, we explicitly formalize this connection. To mathematically characterize this property, we interpret the FIM as a measure of feature coverage.


\begin{restatable}[FIM as Covariance Matrix]{proposition}{fimCovariance}
    \label{prop:fim_covariance}
    Consider the softmax parametrization where $\mathbf{e}_a \in \mathbb{R}^K$ denotes the canonical basis vector for action $a$. The FIM $\fim$ is equivalent to the covariance matrix of the action features $\mathbf{e}_a$ under the policy distribution $\pi_{\bt}$.
\end{restatable}

Under this interpretation, detailed in Appendix~\ref{apx:fim_covariance}, requiring the FIM $\fim$ to be strictly positive definite ensures that the policy maintains non-zero variance along all directions of the simplex. In the context of \sgb with softmax parametrization, this assumption is violated. The standard softmax parametrization (see Equation~\eqref{eq:softmax}) has $K$ linearly dependent parameters, making the FIM singular by construction. To resolve this singularity, for the remainder of the section, we consider a $(K-1)$-dimensional reparametrization of the softmax policy~\cite{metelli2022policy}:\footnote{We refer the reader to Appendix~\ref{apx:log_det_fim} for details about the reparametrized softmax policy.}
\begin{align*}
    \pi_{\bar{\bt}}(a) =
    \begin{cases}
         \frac{e^{\bar{\theta}(a)}}{1+\sum_b e^{\bar{\theta}(b)}} & \text{if } a \ne a_K \\
         1 - \sum_{i=1}^{K-1}\pi_{\bar{\bt}}(a_i) & \text{otherwise}
    \end{cases},
\end{align*} 
where $\bar{\bt} \in \mR^{K-1}$. However, even with the reparametrized policy, the Fisher-non-degeneracy assumption is violated when $\pi_{\bar{\bt}}$ approaches a deterministic policy~\cite{ding2022globaloptimumconvergencemomentumbased}. In fact, when the policy is close to the boundaries of the simplex, the variance of the score function vanishes. Consequently, the smallest eigenvalue of the FIM decays to zero, and $F(\bar{\bt})$ cannot be uniformly lower-bounded by a positive constant $\mu_{\text{F}}$.

\textbf{Spectral Regularization}. To overcome the ill-conditioning of the FIM in \sgb with softmax policy parametrization, rather than relying on Fisher-non-degeneracy or the \npg update rule, we define a Constrained Optimization Problem (COP) that imposes constraints on the FIM's eigenvalues. This COP can be written as: 
\begin{align*}
    \max_{\bar{\bt}\in\mR^{K-1}} J(\bar{\bt})\quad\text{s.t.}\quad\mu_{i}(F(\bar{\bt})) > 0\quad \forall i \in \dsb{K-1}.
\end{align*}
Note that the constraints on the eigenvalues are already implicit in the original problem, since the FIM is strictly positive definite under the reparametrized softmax. To enforce a finite lower bound on the eigenvalues, and to push the FIM away from the boundary of the positive semidefinite cone during the optimization, we induce some slack on the implicit constraints using a \emph{logarithmic barrier regularization}~\citep[Sec. 8.5.3,][]{boyd2004convex}. The resulting regularized objective is defined as:
\begin{align}
    \Phi_{\eta}(\bar{\bt}) = J(\bar{\bt}) + \frac{1}{\eta}\log \text{det}\left(F(\bar{\bt})\right) = J(\bar{\bt}) + \frac{1}{\eta}\sum_{i=1}^K\log \pi_{\bar{\bt}}(a_i),\label{eq:log_barrier_equivalence}
\end{align}
where we used the fact that $\sum_{i=1}^{K-1}\log\mu_i(F(\bar{\bt})) = \log \text{det}F(\bar{\bt})$, and the expression of the determinant of the FIM for the reparametrized softmax policy (a detailed derivation is provided in Appendix~\ref{apx:log_det_fim}). While previous works in the RL literature~\cite{agarwal2021theory, ding2022globaloptimumconvergencemomentumbased, yuan2022general} typically apply a log-barrier directly to the policy as an \emph{ad hoc} heuristic to encourage exploration, our formulation shows that it serves as the explicit structural constraint required to maintain the non-degeneracy of the Fisher geometry. By bounding the FIM eigenvalues away from zero, this approach satisfies the Fisher-non-degeneracy setting by design, providing a geometrically principled exploration mechanism.


This means that we can translate the Fisher-non-degeneracy assumption into an explicit optimization constraint. By bounding the minimum eigenvalue of the FIM away from zero, the algorithm ensures every direction of the feature space ($\bm{e}_a$) is sufficiently covered. This reveals a fundamental connection between the spectral-regularized objective and \npg: both algorithms exploit the Fisher geometry, but through different mechanisms. In \npg, preconditioning the gradient with the inverse FIM causes step sizes to aggressively blow up as the FIM eigenvalues decay near the simplex boundaries, leading to premature convergence. Conversely, the log-barrier approach, rather than inverting a nearly-singular FIM, leverages its eigenvalues as a repulsive penalty that avoids the aggressive, over-committal updates characteristic of \npg. This explicitly restricts the optimization trajectory to the non-degenerate interior regions, safely keeping the policy away from deterministic boundaries. Furthermore, since vanilla PG is known to mimic \npg under Fisher-non-degeneracy~\citep{ding2022globaloptimumconvergencemomentumbased}, this abstract assumption is transformed to a structural requirement for sufficient exploration.


\section{Log-Barrier Stochastic Gradient Bandit}
\label{sec:lb_sgb}

Building upon the spectral-regularization framework established in Section~\ref{sec:comparison}, we now translate these theoretical insights into a more practical algorithmic implementation. Here, we present Log-Barrier Stochastic Gradient Bandit (\lbsgb), a variant of \sgb that uses log-barrier regularization to keep action probabilities away from zero.

\textbf{The Optimization Framework}. While the theoretical derivation in Section~\ref{sec:comparison} relied on a $(K-1)$-dimensional reparametrized policy to establish the non-degeneracy of the FIM, our algorithmic implementation directly optimizes the standard, over-parametrized softmax policy $\pt$ governed by $\bt \in \mR^K$ (Equation~\eqref{eq:softmax}). Because the log-barrier acts as a repulsive penalty evaluating the probabilities of \emph{all} $K$ actions (Equation~\ref{eq:log_barrier_equivalence}), the dimensional reduction to $K-1$ parameters is unnecessary for the practical optimization process. 
Therefore, we define our framework by regularizing the standard performance index $J(\bt)$ with the barrier function $\mathcal{B}_{\eta}(\bt) \coloneqq \frac{1}{\eta}\sum_{a \in \dsb{K}}\log \pi_{\bt}(a)$. The strength of this regularization is governed by a barrier parameter $\eta > 0$, where $1/\eta$ acts as the penalty coefficient. As it will be discussed in Section~\ref{sec:convergence}, $\eta$ should be chosen large enough to control the bias of the learned policy w.r.t. the optimal policy, while still maintaining a sufficient repulsive force to enforce a minimum sampling probability for all arms. The resulting regularized objective is defined as:
\begin{align}
    \barrier \coloneqq J(\bt) + \frac{1}{\eta} \sum_{a \in \dsb{K}} \log \pi_{\bt}(a).
    \label{eq:barrier}
\end{align}
Note that, for any $\eta>0$, the regularized function does attain its maximum in the interior of the simplex (cf. Appendix \ref{apx:maxima}).

\textbf{The \lbsgb Algorithm}. To optimize the regularized objective in Equation~\eqref{eq:barrier}, the \lbsgb algorithm implements the \texttt{SGA} update rule $\bm{\theta}_{t+1} \leftarrow \bt_t + \alpha \widehat{\nabla}_{\bt} \barrierT$, where $\widehat{\nabla}_{\bt} \barrierT$ is the sample-based gradient of Equation~\eqref{eq:barrier}. The complete pseudo-code for \lbsgb is provided in Appendix~\ref{apx:algorithms} (Algorithm~\ref{alg:lb-sgb}). The gradient $\widehat{\nabla}_{\bt} \barrier$ is composed of two distinct terms. The first term represents the stochastic gradient of the performance index $\widehat{\nabla}_{\bt} J(\bt)$, using the unbiased importance sampling estimator $\nabla_{\bt} \left( \pi_{\bt_t}^{\top} \hat{\bm{r}}_t \right)$ presented in Section~\ref{sec:preliminaries}. In contrast, the gradient of the barrier term $\nabla_{\bt}\mathcal{B}_{\eta}(\bt) \coloneqq \frac{1}{\eta}(\bm{1} - K \pt)$ is deterministic, meaning that $\widehat{\nabla}_{\bt} \barrier$ is an unbiased estimator of $\nabla_{\bt} \barrier$ and has the same variance as the \sgb estimator. We refer to Appendix~\ref{apx:lg-sgb} for a complete derivation of $\gradBarrier$.


\textbf{Properties of the Barrier Function}. 
Standard convergence analyses for non-convex optimization often rely on global $L$-smoothness~\cite{mei2022global}; however, the softmax PG exhibits a non-uniform curvature and, to capture this behavior, we leverage the framework of non-uniform smoothness introduced by~\citet{mei2021leveraging}. We first characterize the local curvature of the optimization landscape, showing that the Hessian's spectral radius is adaptively bounded by the norm of the gradient of the regularized objective function.

\begin{restatable}[Non-uniform Smoothness]{lemma}{NUSmoothness}
    \label{lemma:non-uniform-smoothness}
    For all $\bt \in \mR^{K}$, and for all $\bm{r} \in \mR^{K}$ the spectral radius of the Hessian matrix $H(\bt)\in\mR^{K\times K}$ of $\barrier$ is upper bounded by a function of $\bt$. Precisely, for all $\bm{y}\in\mR^{K}$,
    \begin{align*}
        \Big| \bm{y}^{\top} H(\bt) \bm{y} \Big| \leq 3\left(\normGradPhi + \frac{5K}{\eta} \right)\|\bm{y}\|_2^2.
    \end{align*}
\end{restatable}

Lemma~\ref{lemma:non-uniform-smoothness} characterizes the local geometry of the regularized objective. The Hessian bound has two terms: a gradient-dependent term, indicating that the landscape flattens near stationary points, and a baseline curvature $\mathcal{O}\left(K\\/\eta\right)$ induced by the log-barrier. This constant term modifies the optimization geometry to ensure the learned policy remains bounded away from the simplex boundary. Furthermore, the regularization term inevitably shifts the stationary points. Following~\citet{mei2022global}, we present a generalized Łojasiewicz inequality that accounts for this shift.

\begin{restatable}[Weak Non-Uniform Łojasiewicz]{lemma}{Lojasiewicz}
    \label{lemma:NL}
    Under Assumption~\ref{asm:no_ties_r}, we have,
    \begin{align*}
        \normGradPhi \geq \left( \pt(a^*)(r(a^*) - \pt^{\top}\bm{r}) - \frac{K-1}{\eta}\right)^+.
    \end{align*}
\end{restatable}

Lemma~\ref{lemma:NL} establishes a gradient domination condition adapted to the regularized objective. It explicitly relates the gradient norm to the sub-optimality of the policy, showing that the algorithm converges to a neighborhood of the optimal solution as the gradient approaches zero. The size of this neighborhood is determined by the term $K/\eta$, which quantifies the unavoidable bias introduced by the barrier function to maintain feasibility.
\section{Convergence of \lbsgb}
\label{sec:convergence}


In this section, we discuss the convergence of the \lbsgb algorithm. Based on the properties presented in Section~\ref{sec:lb_sgb}, we show the general convergence results under the same assumptions of \citet{mei2024bandit}. 
Then, we derive the worst-case sample complexity of \lbsgb without the assumption on $\pi_{\bt_t}(a^*)$. 
Before stating the convergence results, we begin by establishing the stochastic properties of the \lbsgb gradient estimator. We start showing that the second moment of the stochastic gradient is (partially) controlled by the norm of the true gradient. This result allows us to provide convergence guarantees with constant learning rates.

\begin{restatable}[Self-bounding Property]{lemma}{strongGrowth}
\label{lemma:strong_growth}
    Under Assumption \ref{asm:no_ties_r}, for all $t \geq 1$, almost surely, the stochastic gradient estimator of $\barrierT$ satisfies:
    \begin{align*}
        &\mE_{t}\left[ \sampleNormGradPhiT^2 \right] \leq \frac{16 \Rmax^3 K^{3/2}}{\Delta^2} \normGradPhiT
        + \frac{2K}{\eta} \left( \frac{4K}{\eta} + \frac{16 \Rmax^3 K^{3/2}}{\Delta^2} \right).
    \end{align*}
\end{restatable}

Lemma~\ref{lemma:strong_growth} establishes that the magnitude of the noise vanishes as the gradient approaches zero. In fact, as the policy approaches stationarity, the dominant source of noise diminishes. However, the presence of the bias term $b(\eta)$ is due to regularization and represents a persistent noise floor.

\subsection{Convergence Analysis Assuming Bounded $c^*$}
\label{subsec:bounded_c_star}
We can now characterize the convergence guarantees of \lbsgb by analyzing the evolution of the sub-optimality gap. The following result characterizes the performance of \lbsgb under the same assumptions as~\citet{mei2024bandit}.

\begin{restatable}[Convergence Rate and Sample Complexity]{theorem}{ConvergenceRate}
\label{thr:convergence_rate}
Under Assumption \ref{asm:no_ties_r}, selecting the learning rate $\alpha = \mathcal{O}(1)$, the \lbsgb algorithm guarantees: 
\begin{align*}
    \mE\left[\left(\pi^* - \pi_{\bt_T} \right)^{\top} \bm{r}\right] &\leq \left(1-\frac{1}{2}\sqrt{\frac{2\alpha^2b}{\eta c^*}}\right)^T\mE\left[\left(\pi^* - \pi_{\bt_0} \right)^{\top} \bm{r}\right] + \sqrt{\frac{2bc^*}{\eta}},
\end{align*}
where $b = \mathcal{O}(K)$, and $c^* \coloneqq \sup_{t\geq 0}\mE\left[\frac{1}{\pi_{\bt_t}(a^*)^2}\right] >0 $. In particular, for sufficiently small $\epsilon > 0$ and selecting the barrier parameter as $\eta=\mathcal{O}\left(\epsilon^{-2}\right)$, the number of iterations sufficient to ensure $\mE\left[\left(\pi^* - \pi_{\bt_T} \right)^{\top} \bm{r}\right] \leq \epsilon$ is:\footnote{$\tilde{\mathcal{O}}$ hides logarithmic factors.}
    \begin{align*}
         T = \tilde{O}\left(c^*K^{3/2} \Rmax^3\Delta^{-2} \epsilon^{-1}\right).
    \end{align*}

\end{restatable}

Theorem~\ref{thr:convergence_rate} provides a last-iterate convergence guarantee for \lbsgb to an $\epsilon$-optimal policy under the ``exploration'' assumption. Our approach exploits a recurrence of the instantaneous regret~\cite{montenegro2024learning} that allows controlling the bias term $\beta = \sqrt{2bc^*/\eta}$. This analysis highlights a structural trade-off: while the log-barrier prevents the vanishing gradient problem by ensuring that the optimal action $a^*$ is sampled with sufficient frequency, it simultaneously prevents the policy from learning the optimal deterministic policy. Hence, it is necessary to carefully tune the barrier parameter $\eta$ to ensure sufficient exploration without preventing it from learning the optimal one. In the limit $\eta \rightarrow \infty$, the regularized objective $\barrier$ asymptotically recovers the unregularized performance index $\mathcal{J}(\bt)$, eliminating the bias $\beta$. While Theorem~\ref{thr:convergence_rate} matches the optimal rate of \sgb for a well-behaved $c^*$, Theorem~\ref{thr:worst_case} will show that \lbsgb guarantees convergence even when $c^*$ is potentially unbounded.

\subsection{Worst-case Sample Complexity}
\label{subsec:worst_case}

In the previous section, we established that \lbsgb matches the optimal $\widetilde{\mathcal{O}}\left(\epsilon^{-1}\right)$ sample complexity when the optimization landscape is well behaved, i.e., when $c^*$ is bounded. However, $c^*$ can grow arbitrarily large if the policy is driven toward a suboptimal deterministic vertex of the policy's simplex. Here we show that the log-barrier structurally prevents the policy from collapsing, guaranteeing global convergence without requiring any assumptions on $c^*$. In the following, we provide the sample complexity and regret results for \lbsgb in this setting.

\begin{restatable}[Worst-case Sample Complexity]{theorem}{worstCase}
    \label{thr:worst_case}
    Under Assumption~\ref{asm:no_ties_r}, for a target accuracy $\epsilon>0$, selecting a learning rate $\alpha = \mathcal{O}\left(\Delta^2K^{-7/2}\epsilon^{
    2}\right)$ and the barrier parameter $\eta = 2K\epsilon^{-1}$, the number of iterations $T$ required by the \lbsgb to guarantee $\frac{1}{T}\sum_{t=0}^{T-1} \mE\left[(\pi^* - \pi_{\bt_{t}})^{\top} \bm{r}\right] \leq \epsilon$ is:
    \begin{align*}
    T = \mathcal{O}\left(K^{11/2}\Delta^{-2}\epsilon^{-5}\right).
    \end{align*}
\end{restatable}

The degradation from the optimal $\widetilde{\mathcal{O}}\left(\epsilon^{-1}\right)$ sample complexity to the worst-case $\mathcal{O}\left(\epsilon^{-5}\right)$ rate is due to how the optimization trajectories are handled in the proofs. When $c^*$ is assumed to be bounded, the analysis relies on the policy being well-behaved, which enables a last-iterate convergence guarantee. Without this assumption, it is necessary to explicitly control the iterations in which the policy is near suboptimal vertices of the simplex. Hence, unlike in Theorem~\ref{thr:convergence_rate}, in Theorem~\ref{thr:worst_case} we need to establish an average-iterate convergence guarantee by analyzing the instantaneous regret. This approach is necessary because we need to bound the frequency of ``bad'' iterations across the entire learning process, thereby worsening the sample complexity. In the appendix (Corollary~\ref{cor:regret}), we provide the regret bound for \lbsgb.

\subsection{Discussion}
\label{subsec:discussion}
Table~\ref{tab:discussion} summarizes our theoretical contributions in relation to existing state-of-the-art results. Specifically, assuming $c^* < \infty$, \lbsgb matches the optimal $\tilde{\mathcal{O}}\left(\epsilon^{-1}\right)$ sample complexity result up to logarithmic factors. This demonstrates that our log-barrier mechanism fully preserves the optimal learning efficiency of unregularized methods in favorable landscapes. When this assumption is not met, \lbsgb establishes an $\mathcal{O}\left(\epsilon^{-5}\right)$ sample complexity. We obtain this worst-case guarantee relying solely on the properties of the log-barrier regularization, without requiring any complex algorithmic tweaks. By comparison,~\citet{ding2022globaloptimumconvergencemomentumbased} achieves a $\tilde{\mathcal{O}}\left(\epsilon^{-4.5}\right)$ rate, but their method relies on variance reduction techniques and a decaying learning rate schedule, making a direct comparison with \lbsgb not really possible. Similarly,~\citet{zhang2021sample} require explicit action clipping to achieve their $\mathcal{O}\left(\epsilon^{-6}\right)$ bound. Finally, our result is strictly tighter than the $\mathcal{O}\left(\epsilon^{-6}\right)$ sample complexity established by~\citet{yuan2022general} for the pure log-barrier approach. Our result demonstrates that log-barrier regularization is sufficient to enforce exploration without additional algorithmic modifications. 


\begin{table}[t]
\caption{Comparison of sample complexities for different PG algorithms in the MAB framework. $^\dagger$ We define $c^* \coloneqq \sup_t\mE \left[\pi_{\bt_t}(a^*)^{-2}\right]$.}
\centering
\begin{small}
\begin{tabular}{lccc}
\toprule
Reference & Sample Complexity & Method & Assumption on $c^{*\dagger}$\\
\midrule
\citet{mei2024bandit} & $\mathcal{O}\left(\epsilon^{-1}\right)$ & Vanilla PG (\sgb)& \cmark\\
\citet{mei2022role} & $\mathcal{O}\left(\epsilon^{-1}\right)$ & \npg with Baseline& \cmark\\
\textbf{This work}  & $\bm{\widetilde{\mathcal{O}}\left(\epsilon^{-1}\right)}$ & \textbf{Log-barrier (\lbsgb)}  & \cmark\\
\midrule
\citet{zhang2021sample}  & $\mathcal{O}\left(\epsilon^{-6}\right)$ & Log-barrier, Clipping  &   \xmark\\
\citet{yuan2022general}  & $\mathcal{O}\left(\epsilon^{-6}\right)$ & Log-barrier  &   \xmark\\
\textbf{This work}  & $\bm{\mathcal{O}\left(\epsilon^{-5}\right)}$ & \textbf{Log-barrier (\lbsgb)} & \xmark\\
\bottomrule
\end{tabular}
\end{small}
\label{tab:discussion}
\end{table}
\begin{figure}[t]
    \centering
    \makebox[\textwidth][c]{
        \begin{tabular}{@{}c@{\hspace{0.2em}}c@{\hspace{0.5em}}c@{\hspace{0.5em}}c@{}}
            
            & $K=10$ & $K=100$ & $K=1000$ \\
            \noalign{\vspace{0.5em}}
            
            \rotatebox[origin=c]{90}{$\alpha = 0.01$} &
            \begin{subfigure}[c]{0.315\textwidth}
                \includegraphics[width=\textwidth]{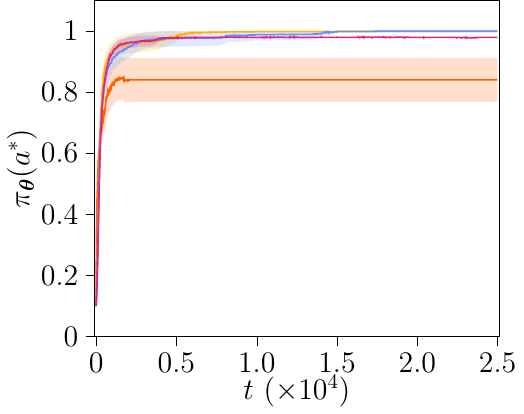} 
            \end{subfigure} &
            \begin{subfigure}[c]{0.315\textwidth}
                \includegraphics[width=\textwidth]{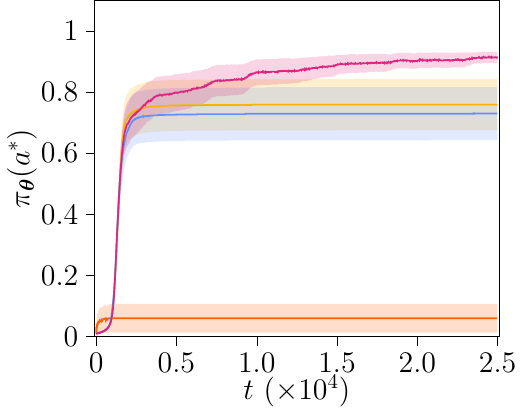}
            \end{subfigure} &
            \begin{subfigure}[c]{0.315\textwidth}
                \includegraphics[width=\textwidth]{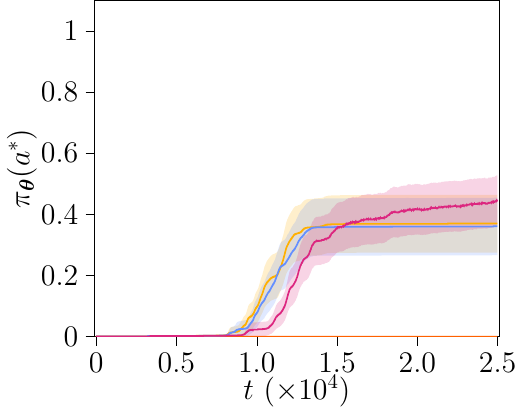}
            \end{subfigure} \\
            
            \noalign{\vspace{1.5em}}
            
            \rotatebox[origin=c]{90}{$\alpha = 0.1$} &
            \begin{subfigure}[c]{0.315\textwidth}
                \includegraphics[width=\textwidth]{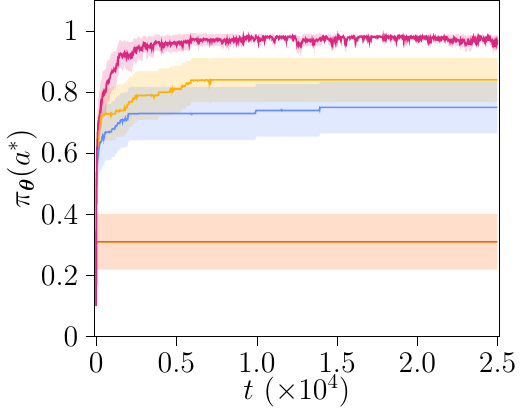} 
            \end{subfigure} &
            \begin{subfigure}[c]{0.315\textwidth}
                \includegraphics[width=\textwidth]{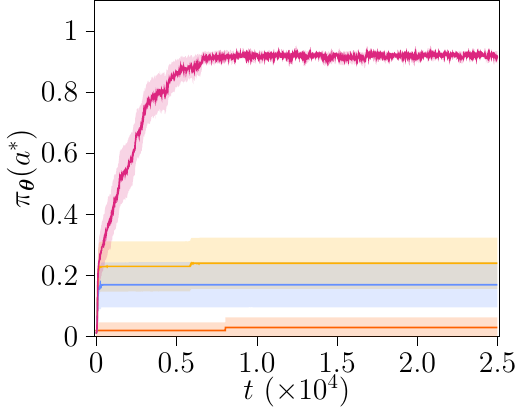}
            \end{subfigure} &
            \begin{subfigure}[c]{0.315\textwidth}
                \includegraphics[width=\textwidth]{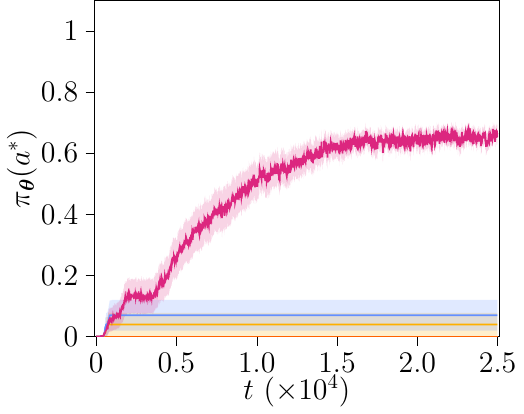}
            \end{subfigure} \\
            
        \end{tabular}
    }
    
    \vspace{1em} 
    
    \centerline{
        \legendline{lbsgb} \lbsgb \hspace{1.5em}
        \legendline{sgb} \sgb \hspace{1.5em}
        \legendline{ent} \ent \hspace{1.5em}
        \legendline{npg} \npg    
    }
    
    \vspace{0.5em} 
    
    \caption{Comparison between algorithms across $K=\{10, 100, 1000\}$. The top row uses a learning rate $\alpha = 0.01$, and the bottom row uses $\alpha = 0.1$. All experiments use $\Delta^* = 0.1$ (100 runs $\pm$ 95\% C.I.).}
    \label{fig:experiments_grid}
\end{figure}

\section{Experimental Results}
\label{sec:experiments}

In this section, we empirically validate the theoretical results derived in Section~\ref{sec:convergence}. We evaluate the performance of \lbsgb, comparing it against \sgb, its entropy-regularized version (\ent), and \npg. Specifically, we analyze the learning dynamics by tracking how the probability of the optimal action, $\pi_{\bt_t}(a^*)$, evolves over time. The results are aggregated over $N = 100$ independent runs. In the visualizations, the solid curves represent the sample mean, while the shaded regions denote the $95\%$ confidence intervals computed using $t$-intervals. Across all experiments, we fix the sub-optimality gap to $\Delta^* = 0.1$ and the maximum reward to $\Rmax = 1$. The exploration parameters for both \lbsgb and \ent are set according to the number of arms and learning rate. The number of steps per experiment is $T = 2.5\cdot10^4$. 
Figure~\ref{fig:experiments_grid} illustrates experiments for varying action space dimensions $K \in \{10, 100, 1000\}$. In the top row, the learning rate is $\alpha = 0.01$, while in the bottom row the learning rate is $\alpha = 0.1$. As the number of arms $K$ increases, the baselines (\sgb, \ent, and \npg) exhibit performance degradation, failing to learn the optimal arm. In contrast, \lbsgb consistently converges toward the optimal action, showing that it is significantly more robust with respect to the learning rate $\alpha$ and scales better with $K$. We refer to Appendix~\ref{apx:experimental_details} for additional results.
\section{Conclusion}
\label{sec:conclusions}

In this work, we showed that \lbsgb matches the optimal sample complexity in the gradient bandit setting when assumptions on the learning dynamics hold, demonstrating that the log-barrier regularization does not worsen the convergence of the vanilla algorithm. However, the convergence rate deteriorates in worst-case scenarios, where the policy must recover from vanishing action probabilities. We believe that implementing adaptive hyperparameters, such as dynamically tuning the barrier parameter $\eta$ throughout the learning process, would enable the algorithm to achieve faster rates. A fundamental limitation of our current framework lies in its extension to general RL. In the MAB setting, our formulation establishes that a log-barrier on the parametric policy corresponds to a spectral regularization of the FIM. However, in the RL setting, this equivalence no longer holds. Because in RL the FIM in Markov Decision Processes depends on both states and actions, true spectral regularization must account for the visitation dynamics of the joint space. In this context, applying a log barrier only to the policy ignores the state distribution entirely, rendering it a state-wise approximation of the true spectral regularization. Consequently, an interesting direction for future work would be to develop a provably correct and computationally tractable method to extend spectral regularization to the RL framework. This work thereby provides a theoretical foundation for PG-based RL methods with improved exploration.


\bibliographystyle{unsrtnat}
\bibliography{bibliography}

\newpage
\appendix
\onecolumn

\section{Detailed Related Works}
\label{apx:detailed_related_works}

\begin{table*}[t]
\caption{Comparison of convergence rates and learning rates for different PG algorithms. 
$^\dagger$ We define $c^* \coloneqq \sup_t\mE \left[\pi_{\bt_t}(a^*)^{-2}\right]$. $^\ddagger$ In~\citet{lu2024towards}, the authors show that \sgb achieves $\mathcal{O}\left(\epsilon^{-1}\right)$ or $\mathcal{O}\left(\epsilon^{-3}\right)$ depending on whether the initial learning rate is tuned to meet the strong growth condition or not. Nonetheless, it seems that their analysis has a flaw when computing the learning rate in the proof of their Theorem 5.}
\centering
\begin{small}
\resizebox{\textwidth}{!}{
\begin{tabular}{lcccc}
\toprule
Reference & Sample Complexity & Learning Rate $\alpha$ & Method & Assumptions\\
\midrule
\citet{mei2024bandit} & $\mathcal{O}\left(\epsilon^{-1}\right)$ & $\mathcal{O}\left(1\right)$ & Vanilla PG (\sgb)& No ties, $c^{*}<\infty^\dagger$\\
\citet{lu2024towards} &  $\mathcal{O}\left(\epsilon^{-3}\right)^\ddagger$ & $\mathcal{O}\left(\alpha^t\right)$ & Vanilla PG & $c^{*}<\infty$ \\
\textbf{This work}  & $\bm{\widetilde{\mathcal{O}}\left(\epsilon^{-1}\right)}$ & $\bm{\mathcal{O}\left(1\right)}$ & \textbf{Log-barrier} & \textbf{No ties, $\bm{c^{*}<\infty}$}\\
\midrule
\citet{ding2021beyond}  & $\mathcal{O}\left(\epsilon^{-2}\right)$ & $\mathcal{O}\left(1/t\right)$ & Entropy regularization&  $\inf_{t}\min_a \pi_{\bt}(a)>0$  \\
\citet{cen2021fastglobalconvergencenatural}  & $\mathcal{O}\left(\epsilon^{-2}\right)$ & $\mathcal{O}\left(1/\log t\right)$ & Entropy regularized \npg  &  Generative model \\
\citet{zhang2021sample}  & $\mathcal{O}\left(\epsilon^{-6}\right)$ & $\mathcal{O}\left(1/\log t\right)$ & Log-barrier, Clipping  &   \\
\citet{yuan2022general}  & $\widetilde{\mathcal{O}}\left(\epsilon^{-6}\right)$ & $\mathcal{O}\left(1/ t\right)$ & Log-barrier &    \\
\citet{ding2022globaloptimumconvergencemomentumbased} & $\mathcal{O}\left(\epsilon^{-4.5}\right)$ & $\mathcal{O}\left(1/\sqrt{t}\right)$ & Log-barrier, Momentum &  \\
\textbf{This work}  & $\bm{\mathcal{O}\left(\epsilon^{-5}\right)}$ & $\bm{\mathcal{O}\left(\epsilon^{2}\right)}$ & \textbf{Log-barrier} & \textbf{No ties}\\
\bottomrule
\end{tabular}}
\end{small}

\label{tab:convergence-comparison}
\end{table*}
In this section, we review the most relevant literature, focusing on sample-complexity and regret results for MABs and Markov Decision Processes (MDPs). In Table~\ref{tab:convergence-comparison}, we compare some of the works discussed in this section.

\textbf{Sample Complexity of \sgb}. The core analytical difficulty of \sgb lies in the geometry of the softmax policy, which, while satisfying smoothness assumptions, only satisfies a weak version of the Łojasiewicz inequality~\cite{mei2022global}. Furthermore, while it is standard practice to establish convergence to a stationary point under the assumption of bounded estimator variance and employing a decaying learning rate~\cite{robbins1951stochastic, zhang2020global}, converging to a stationary point is a weak requirement in the context of bandits. Since the gradient of the objective vanishes for any deterministic policy, the set of stationary points includes all vertices of the policy simplex, including all suboptimal deterministic policies~\cite{mei2024bandit}. This property implies that the gradient signal is not globally well-behaved and, as the sampling probability of the optimal arm approaches zero, the gradient signal vanishes, converging to a suboptimal policy~\cite{mei2020escaping}. Despite these difficulties, convergence to a global optimum is guaranteed~\cite{mei2024bandit}, even with no requirements on the learning rate $\alpha$~\cite{mei2024small}, establishing that \sgb needs $\mathcal{O}\left(\epsilon^{-1}\right)$ iterations to learn an $\epsilon$-optimal policy. However, as noted by~\citet{baudrydoes}, this sample complexity guarantee may become vacuous due to a hidden assumption, as the constant governing the $\mathcal{O}\left(\epsilon^{-1}\right)$ term depends on the inverse second moment of the probability assigned to the optimal arm. Furthermore,~\citet{mei2022role} showed that Natural PG (\npg,~\citet{kakade2001natural}) with an appropriate baseline also achieves  $\mathcal{O}\left(\epsilon^{-1}\right)$ sample complexity. Finally,~\citet{lu2024towards} proved that PG, in the gradient bandit setting, with exponentially decaying learning rate, achieves $\mathcal{O}\left(\epsilon^{-1}\right)$ or $\mathcal{O}\left(\epsilon^{-3}\right)$ depending on whether the initial learning rate is tuned to meet the strong growth condition or not.

\textbf{Regret of \sgb}. Unlike the theoretical analysis of traditional bandit algorithms, which exploit arm-specific statistics, \sgb is usually analyzed via trajectories determined by the global history of the reward received by the algorithm, and providing regret guarantees requires significant effort. Recently,~\citet{mei2024bandit} 
showed that with a constant learning rate $\alpha$, the \sgb algorithm can achieve $\mathcal{O}\left(\log T\right)$ regret. 
The robustness of this result has been questioned by \citet{baudrydoes}, who point out the existence of low-probability events that may make the regret linear under the current analysis. Consequently,~\citet{baudrydoes} provided an alternative proof of logarithmic regret in the $2$-armed case, conjecturing an extension to the $K$-armed setting. 

\textbf{Sample Complexity of PG in MDPs}.\footnote{Here, we distinguish between iteration complexity, defined as the number of policy updates required to reach an $\epsilon$-optimal policy, and sample complexity, which counts the total number of environment interactions. In the MAB setting, where each iteration involves exactly one sample to compute the stochastic gradient, these two measures are equivalent.} %
MDPs introduce additional complexities since it is necessary to consider the effect of the policy $\pi_{\bt}$ on the state distribution. While global convergence results for PG with exact gradient information in a tabular setting have been proved~\cite{mei2022global, agarwal2021theory}, applying these techniques in stochastic environments requires addressing the variance of trajectory-based estimators. Recent literature in the stochastic setting has shown that PG methods can achieve $\widetilde{\mathcal{O}}\left(\epsilon^{-2}\right)$ sample complexity via entropy regularization and decaying learning rates~\cite{ding2021beyond}. Similarly, entropy-regularized \npg achieves $\widetilde{\mathcal{O}}\left(\epsilon^{-2}\right)$ sample complexity, provided the regularization parameter is appropriately tuned~\cite{cen2021fastglobalconvergencenatural}. \citet{zhang2021sample} proved $\widetilde{\mathcal{O}}\left(\epsilon^{-6}\right)$ sample complexity and $\mathcal{O}\left(T^{5/6}\right)$ regret for the \texttt{REINFORCE} algorithm with log-barrier regularization, and~\citet{yuan2022general} proved that PG with log-barrier achieves $\widetilde{\mathcal{O}}\left(\epsilon^{-6}\right)$.~\citet{ding2022globaloptimumconvergencemomentumbased} demonstrated $\widetilde{\mathcal{O}}\left(\epsilon^{-9/2}\right)$ sample complexity for momentum-based PG and log-barrier regularization. Finally,~\citet{labbi2026beyond} proved that PG with $f$-divergence regularization achieves $\widetilde{\mathcal{O}}\left(\epsilon^{-1}\right)$ assuming non-vanishing action probabilities. Table~\ref{tab:convergence-comparison} summarizes the related works, providing a comparison in terms of sample complexity and required assumptions.
\section{Algorithmic Details}
\label{apx:algorithms}

\subsection{Stochastic Gradient Bandit}
In this section we report the \texttt{SGB} algorithm as it is reported in \cite{mei2024bandit}. In particular, we show the pseudo-code (Algorithm~\ref{alg:gradient_bandit}) in its basic version without baseline.

\begin{algorithm}[t]
  \caption{Stochastic Gradient Bandit (\texttt{SGB})}
  \label{alg:gradient_bandit}
  \begin{algorithmic}[1]
    \STATE {\bfseries Input:} Iterations $T$; step size $\alpha > 0$; action set $\dsb{K}$
    \STATE {\bfseries Initialize:} Parameters $\bt(a) \leftarrow 0$ for all $a \in \dsb{K}$, $t \leftarrow 1$
    \FOR{$t \in \dsb{T}$}
        \STATE Compute policy $\pi_t(a) \leftarrow \frac{e^{\bt_{t}(a)}}{\sum_{b \in \dsb{K}} e^{\bt_{t}(b)}}$ for all $a \in \dsb{K}$
        \STATE Select action $a_t \sim \pi_t$
        \STATE Observe reward $R_t(a_t)$
        \FOR{$a \in \dsb{K}$}
            \IF{$a = a_t$}
                \STATE $\bt(a) \leftarrow \bt(a) + \alpha \cdot (1 - \pi_t(a))R_t(a_t)$
            \ELSE
                \STATE $\bt(a) \leftarrow \bt(a) - \alpha \cdot \pi_t(a)R_t(a_t)$
            \ENDIF
        \ENDFOR
    \ENDFOR
  \end{algorithmic}
\end{algorithm}

\subsection{Stochastic Gradient Bandit with Entropy Regularization}
In this section we present the \texttt{SGB} algorithm with entropy regularization~\cite{mei2022global, agarwal2021theory}. The algorithm augments the instantaneous reward $R_t(a_t)$ with the (empirical) entropy regularization term $-\tau \log \pi_{\bt_t}(a_t)$, where $\tau$ is a hyper-parameter, which acts as an intrinsic exploration bonus. Algorithm~\ref{alg:gradient_bandit_ent} shows the pseudocode for \texttt{SGB} with entropy regularization.
\begin{algorithm}[t]
  \caption{\texttt{SGB} with Entropy Regularization}
  \label{alg:gradient_bandit_ent}
  \begin{algorithmic}[1]
    \STATE {\bfseries Input:} Iterations $T$; step size $\alpha > 0$; temperature $\tau \ge 0$; action set $\dsb{K}$
    \STATE {\bfseries Initialize:} Parameters $\bt(a) \leftarrow 0$ for all $a \in \dsb{K}$, $t \leftarrow 1$
    \FOR{$t \in \dsb{T}$}
        \STATE Compute policy $\pi_t(a) \leftarrow \frac{e^{\bt_{t}(a)}}{\sum_{b \in \dsb{K}} e^{\bt_{t}(b)}}$ for all $a \in \dsb{K}$
        \STATE Select action $a_t \sim \pi_t$
        \STATE Observe reward $R_t(a_t)$
        \STATE Compute regularized reward $\hat{r}_t \leftarrow R_t(a_t) - \tau \log \pi_{\bt_t}(a_t)$
        \FOR{$a \in \dsb{K}$}
            \IF{$a = a_t$}
                \STATE $\bt(a) \leftarrow \bt(a) + \alpha \cdot (1 - \pi_t(a))\hat{r}_t$
            \ELSE
                \STATE $\bt(a) \leftarrow \bt(a) - \alpha \cdot \pi_t(a)\hat{r}_t$
            \ENDIF
        \ENDFOR
    \ENDFOR
  \end{algorithmic}
\end{algorithm}

\subsection{Natural Policy Gradient}
\label{apx:npg_bandit}
In this section, we present the Natural Policy Gradient (\npg) variant of \texttt{SGB}, following the derivation of \cite{chung2021beyond}. The \npg updates the policy parameters $\bt$  along the steepest direction by preconditioning the standard gradient with the inverse Fisher Information Matrix, $\fim^{-1} \nabla_{\bt}\mathcal{J}(\bt)$, where $\fim$ serves as a metric tensor which accounts for the curvature of the policy space~\cite{kakade2001natural}. In \texttt{SGB}, the Fisher information matrix is defined as:
\begin{align*}
\fim = \mE_{a\sim \pt}\left[ \nabla_{\bt} \log\pi_{\bt}(a)\nabla_{\bt} \log\pi_{\bt}(a)^{\top} \right] = \text{diag}(\pi_{\bt}) - \pi_{\bt} \pi_{\bt}^{\top}, 
\end{align*}
which is singular, hence not invertible. Hence, we cannot directly compute the update direction inverting the FIM, but we can solve the linear system,
\begin{align*}
\fim\cdot x = \mathbf{e}_{a_t} -\pt
\end{align*}
obtaining:
\begin{align*}
x = \frac{1}{\pi_{\bt}(a_t)}\mathbf{e}_{a_t} +\beta\bm{1},
\end{align*}
where $\mathbf{e}_{a_t}$ is an indicator vector defined as the $a_t$-th standard basis vector, $\bm{1}$ is the all-ones vector, and $\beta \in \mR$ is a scalar parameter representing the null space of $\fim$. Since the softmax parametrization is shift-invariant, the choice of $\beta$ does not affect the resulting policy, hence we can select $\beta = 0$ without loss of generality. Substituting the natural gradient direction $x$ into the stochastic gradient ascent update, we obtain:
\begin{align*}
    \bt_{t+1} \leftarrow \bt_{t} + \alpha R_t(a_t) \left( \frac{1}{\pi_{\bt_t}(a_t)}\mathbf{e}_{a_t}\right).
\end{align*}

Algorithm~\ref{alg:npg_bandit} shows the pseudo-code for Natural \texttt{SGB}. Unlike the standard \texttt{SGB} update, which updates all arms at every step, increasing the probability of the selected arm and decreasing the others, the \npg update is sparse, updating only the selected arm.
\begin{algorithm}[t]
  \caption{Natural Policy Gradient (\npg)}
  \label{alg:npg_bandit}
  \begin{algorithmic}[1]
    \STATE {\bfseries Input:} Iterations $T$; step size \(\alpha > 0\); action set \(\dsb{K}\)
    \STATE {\bfseries Initialize:} Parameters \(\bt(a) \leftarrow 0\) for all \(a \in \dsb{K}\), \(t \leftarrow 1\)
    \FOR{\(t \in \dsb{T}\)}
        \STATE Compute policy \(\pi_t(a) \leftarrow \frac{e^{\bt_{t}(a)}}{\sum_{b \in \dsb{K}} e^{\bt_{t}(b)}}\) for all \(a \in \dsb{K}\)
        \STATE Select action \(a_t \sim \pi_t\)
        \STATE Observe reward \(R_t(a_t)\)
        \STATE \(\bt(a_t) \leftarrow \bt(a_t) + \alpha \cdot \frac{R_t(a_t)}{\pi_t(a_t)}\)
    \ENDFOR
  \end{algorithmic}
\end{algorithm}

\subsubsection{FIM as Covariance Matrix}
\label{apx:fim_covariance}
In the following proposition, we show that we can interpret the FIM as a covariance matrix of action features. 

\fimCovariance*
\begin{proof}
    First, observe that the expected value of the feature vectors $\mathbf{e}_a$ is the policy vector itself:
    \begin{align*}
        \mE_{a \sim \pi_{\bt}} [\mathbf{e}_a] = \sum_{a \in \mathcal{A}} \pi_{\bt}(a) \mathbf{e}_a = \pi_{\bt}.
    \end{align*}
    For the softmax parametrization, the score function is given by the centered feature vector:
    \begin{align*}
        \nabla_{\bt} \log \pi_{\bt}(a) = \mathbf{e}_a - \pi_{\bt}.
    \end{align*}
    Substituting this into the definition of the Fisher Information Matrix, we obtain:
    \begin{align*}
        \fim &= \mE_{a \sim \pi_{\bt}} \left[ \nabla_{\bt} \log \pi_{\bt}(a) \nabla_{\bt} \log \pi_{\bt}(a)^\top \right] \\
        &= \mE_{a \sim \pi_{\bt}} \left[ (\mathbf{e}_a - \pi_{\bt})(\mathbf{e}_a - \pi_{\bt})^\top \right] \\
        &= \mE_{a \sim \pi_{\bt}} \left[ (\mathbf{e}_a - \mE[\mathbf{e}_a])(\mathbf{e}_a - \mE[\mathbf{e}_a])^\top \right].
    \end{align*}
    This coincides with the definition of the covariance matrix $\text{Cov}_{\pi_{\bt}}(\mathbf{e}_a)$.
\end{proof}

\subsubsection{Log-Determinant of the FIM}
\label{apx:log_det_fim}
To explicitly connect the spectral constraint to the log-barrier formulation, we compute the determinant of the Fisher Information Matrix in the reparametrized space. Using the softmax policy with the parameter $\bar{\bt} \in \mR^{K-1}$, hence,
\begin{align}
    \pi_{\bar{\bt}}(a) =
    \begin{cases}
        \frac{e^{\bar{\theta}(a)}}{1+\sum_b e^{\bar{\theta}(b)}} & \text{if } a \ne a_K \\
        1 - \sum_{i=1}^{K-1}\pi_{\bar{\bt}}(a_i) & \text{otherwise}
    \end{cases},
\end{align}
the FIM is a $(K-1) \times (K-1)$ matrix given by $F(\bar{\bt}) = \text{diag}(\pi_{\bar{\bt}}) - \pi_{\bar{\bt}} \pi_{\bar{\bt}}^\top$, where, with slight abuse of notation, $\pi_{\bar{\bt}}$ denotes the vector of the first $K-1$ action probabilities.

We compute the determinant using the Matrix Determinant Lemma, which states that for a diagonal matrix $\mathbf{D}$ and vector $\mathbf{v}$, $\det(\mathbf{D} - \mathbf{v}\mathbf{v}^\top) = \det(\mathbf{D})(1 - \mathbf{v}^\top \mathbf{D}^{-1} \mathbf{v})$. Setting $\mathbf{D} = \text{diag}(\pi_{\bar{\bt}})$ and $\mathbf{v} = \pi_{\bar{\bt}}$, we obtain:
\begin{align*}
    \det(F(\bar{\bt})) &= \det(\text{diag}(\pi_{\bar{\bt}})) \left( 1 - \pi_{\bar{\bt}}^\top \text{diag}(\pi_{\bar{\bt}})^{-1} \pi_{\bar{\bt}} \right) \nonumber \\
    &= \left( \prod_{i=1}^{K-1} \pi_i \right) \left( 1 - \sum_{i=1}^{K-1} \pi_i \frac{\pi_i}{\pi_i} \right) \nonumber \\
    &= \left( \prod_{i=1}^{K-1} \pi_i \right) \left( 1 - \sum_{i=1}^{K-1} \pi_i \right).
\end{align*}
Then, the term in the second parenthesis simplifies exactly to the probability of the last arm, $1 - \sum_{i=1}^{K-1} \pi_i = \pi_K$ since $\sum_{i=1}^K \pi_i = 1$. Consequently, the determinant becomes the product of the probabilities of all arms:
\begin{align*}
    \det(F(\bar{\bt})) = \left( \prod_{i=1}^{K-1} \pi_i \right) \cdot \pi_K = \prod_{i=1}^{K} \pi_i.
\end{align*}
Finally, substituting this result into the log-determinant term yields:
\begin{align*}
    \log \det(F(\bar{\bt})) = \sum_{i=1}^{K-1} \log \pi_{\bar{\bt}}(a_i) + \log\left(1 - \sum_{i=1}^{K-1}\pi_{\bar{\bt}}(a_i)\right).
\end{align*}
If we rewrite this regularization term with the softmax parametrization in Equation~\eqref{eq:softmax}, we obtain exactly the objective function in Equation~\eqref{eq:barrier}.
This derivation confirms that constraining the log-determinant of the FIM is mathematically equivalent to the sum-log-barrier regularization employed in \lbsgb.

\subsection{Log-barrier Stochastic Gradient Bandit}
\label{apx:lg-sgb}
In this section, we present the details of the \lbsgb algorithm, whose pseudocode is shown in Algorithm~\ref{alg:lb-sgb}.

\begin{algorithm}[H]
  \caption{Log-barrier Stochastic Gradient Bandit (\lbsgb)}
  \label{alg:lb-sgb}
  \begin{algorithmic}[1]
    \STATE {\bfseries Input:} Iterations $T$; step size $\alpha > 0$; barrier parameter $\eta > 0$; action set $\dsb{K}$
    \STATE {\bfseries Initialize:} Parameters $\bt(a) \leftarrow 0$ for all $a \in \dsb{K}$, $t \leftarrow 1$
    \FOR{$t \in \dsb{T}$}
        \STATE Compute policy $\pi_t(a) \leftarrow \frac{e^{\bt_{t}(a)}}{\sum_{b \in \dsb{K}} e^{\bt_{t}(b)}}$ for all $a \in \dsb{K}$
        \STATE Select action $a_t \sim \pi_t$
        \STATE Observe reward $R_t(a_t)$
        \FOR{$a \in \dsb{K}$}
            \IF{$a = a_t$}
                \STATE $\bt(a) \leftarrow \bt(a) + \alpha \cdot\left( (1 - \pi_t(a))R_t(a_t) + \frac{1}{\eta} (1 - K \cdot\pi_t(a))\right)$
            \ELSE
                \STATE $\bt(a) \leftarrow \bt(a) - \alpha \cdot \left(\pi_t(a)R_t(a_t) - \frac{1}{\eta} (1 - K \cdot\pi_t(a)) \right)$
            \ENDIF
        \ENDFOR
    \ENDFOR
  \end{algorithmic}
\end{algorithm}

In the following, we derive the gradient and Hessian of the \lbsgb objective. We recall the regularized performance index defined in Equation~\eqref{eq:barrier}:
\begin{align}
    \barrier \coloneqq J(\bt) + \frac{1}{\eta} \sum_{a \in \dsb{K}} \log \pi_{\bt}(a).
\end{align}

To compute the gradient of the barrier term, we first need to compute the following quantities:
\begin{align*}
    \nabla_{\bt}\log \pt(a_i) &= \nabla_{\bt}\log \frac{e^{\theta_i}}{\sum_{j=1}^{K}e^{\theta_j}} = \mathbf{e}_i - \pt \\
    \nabla_{\bt}\left(\sum_{i=1}^{K}\log\pt(a_i) \right)&= \sum_{i=1}^{K}\left(\mathbf{e}_i - \pt \right) = \mathbf{1} - K\pt,
\end{align*}
where $\mathbf{e}_i$ is the vector with all zeros and a one in the $i$-th position, and $\mathbf{1}$ is the $K$-dimensional ones vector. Then, the gradient of the barrier function is:
\begin{align*}
    \nabla_{\bt}\barrier = (\text{diag}(\pi_{\bt}) - \pi_{\bt} \pi_{\bt}^{\top}) \bm{r} + \frac{1}{\eta}\left(\mathbf{1} - K\pt\right) = A(\bt) + \frac{1}{\eta}B(\bt).
\end{align*}

Let's now consider the Hessian of the regularized function:
\begin{align*}
    H(\bt) = \frac{\partial}{\partial \bt}\left( A(\bt) + \frac{1}{\eta}B(\bt) \right),
\end{align*}
where we consider the two terms separately. For the first term, we follow the proof of Lemma 2 in \cite{mei2021leveraging}, obtaining:
\begin{align}
    \label{eq:hessian_performance}
    \left[\frac{\partial}{\partial \bt} A(\bt)\right]_{i, j} = \delta_{ij}\pi_i(r_i - \pt^{\top}r) - \pi_i\pi_j(r_i - \pt^{\top}r) - \pi_i\pi_j(r_j - \pt^{\top}r),
\end{align}
where $\delta_{ij}=\begin{cases}&1, \quad\text{if}~i=j\\&0,\quad\text{otherwise} \end{cases}$ is the Kronecker delta. The second term, representing the Jacobian of the restorative force $B(\bt)$, can be written as:
\begin{align}
    \label{eq:hessian_barrier}
    \left[\frac{\partial}{\partial \bt} B(\bt)\right]_{i, j} = \frac{\partial}{\partial \theta_j}\left( 1-K\pi_i \right) = -K(\delta_{ij}\pi_i - \pi_i\pi_j).
\end{align}
Finally, the $ij$-th entry of the Hessian of the regularized objective is:
\begin{align}
    \label{eq:barrier_hessian}
    \left[H(\bt)\right]_{i, j} = \delta_{ij}\pi_i(r_i - \pt^{\top}r) - \pi_i\pi_j(r_i + r_j - 2\pt^{\top}r) - \frac{K}{\eta}\left(\delta_{ij}\pi_i - \pi_i\pi_j\right).
\end{align}
\section{Proofs}
\label{apx:proofs}
In this appendix, the proofs for the Lemmas and Theorems presented in Sections~\ref{sec:lb_sgb} and~\ref{sec:convergence} are provided.

\subsection{Proofs for Section~\ref{sec:lb_sgb}}
\label{apx:proof_sec_lb_sgb}
Here, we provide the proofs for the properties of the \lbsgb algorithm provided in Section~\ref{sec:lb_sgb}. Lemma~\ref{lemma:non-uniform-smoothness} describes the boundedness of the spectral radius of the barrier's Hessian, and Lemma~\ref{lemma:non_uniform_smoothness_iterates} generalizes the notion of non-uniform smoothness of the barrier between consecutive learning iterates. Lemma~\ref{lemma:ub_sample_gradient_norm} provides an upper bound for the sample gradient norm. Finally, Lemma~\ref{lemma:NL} provides the notion of non-uniform Łojasiewicz for our objective $\barrier$.

\NUSmoothness*
\begin{proof}
    Let $H(\bt)$ be the Hessian matrix of the barrier function $\barrier$, where each entry of this matrix is described by line~\eqref{eq:barrier_hessian}. To show that $\barrier$ is non-uniformly smooth, we need to show that the spectral radius of $H(\bt)$ is upper-bounded. For this purpose, pick $y\in\mR^{K}$. Then

    \begin{align*}
        \left| y^{\top}H(\bt)y \right| &= \sum_{i=1}^{K}\sum_{j=1}^{K}y_i\left[H(\bt)\right]_{i, j}y_j\\
        & = \sum_{i=1}^{K}\sum_{j=1}^{K}y_i\left(\left[A(\bt)\right]_{i, j}+\left[B(\bt)\right]_{i, j}\right)y_j
    \end{align*}
    Where the terms $A(\bt)$ and $B(\bt)$ are defined as:
    \begin{align*}
         A(\bt) \coloneqq \delta_{ij}\pi_i(r_i - \pt^{\top}\bm{r}) - \pi_i\pi_j(r_i + r_j - 2\pt^{\top}\bm{r}), \quad B(\bt) = - \frac{K}{\eta}\left(\delta_{ij}\pi_i - \pi_i\pi_j\right),
    \end{align*}
    See Appendix~\ref{apx:algorithms} for a detailed derivation of the Hessian of the barrier function $\barrier$. Now, analyzing the two terms separately:
    \begin{align*}
        \sum_{i=1}^{K}\sum_{j=1}^{K}y_i\left[A(\bt)\right]_{i, j}y_j = \left(\fim^{\top}\bm{r} \right)(y \odot y) - 2 \left(\fim^{\top}\bm{r} \right)y(\pt^{\top}y),
    \end{align*}
    from the proof of Lemma 2 in \cite{mei2021leveraging}. Considering the second term:
    \begin{align*}
        \sum_{i=1}^{K}\sum_{j=1}^{K}y_i\left[B(\bt)\right]_{i, j}y_j &= \sum_{i=1}^{K}\sum_{j=1}^{K}y_i\left[-K(\delta_{ij}\pi_i - \pi_i\pi_j)\right]_{i, j}y_j \\
        &=-K y^{\top} \fim y\\
        &=-K y^{\top}\left( \text{diag}(\pt) - \pt\pt^{\top} \right)y\\
        &=-K\left( y^{\top}\text{diag}(\pt)y - y^{\top}\pt\pt^{\top}y \right)\\
        &= -K\left(\pt^{\top}(y \odot y) - \pt^{\top}y\pt^{\top}y \right)
    \end{align*}

    Putting all together:
    \begin{align*}
        \left| y^{\top}H(\bt)y \right| &= \left| \left(\fim^{\top}\bm{r} \right)^{\top}(y \odot y) - 2 \left(\fim^{\top}\bm{r} \right)^{\top}y(\pt^{\top}y) -\frac{K}{\eta}\left(\pt^{\top}(y \odot y) - \pt^{\top}y(\pt^{\top}y) \right)\right| \\
        &= \left| \left(\left(\fim^{\top}\bm{r} \right)^{\top} - \frac{K}{\eta}\pt^{\top}\right)(y \odot y) - \left(2\left(\fim^{\top}\bm{r} \right)^{\top} + \frac{K}{\eta}\pt^{\top}\right)y(\pt^{\top}y) \right|\\
        &= \Bigg| \left(\left(\fim^{\top}\bm{r} \right)^{\top} + \frac{1}{\eta}\left(\bm{1} - K\pt^{\top}\right) - \frac{1}{\eta}\bm{1}\right)(y \odot y)  \\ & \qquad -\left(2\left(\fim^{\top}\bm{r} \right)^{\top} + \frac{2}{\eta}\left(\bm{1} - K\pt^{\top}\right)- \frac{1}{\eta}\bm{1} + \frac{3K}{\eta}\pt^{\top}\right)y(\pt^{\top}y)\Bigg|\\
        & \leq \left\| \left(\fim^{\top}\bm{r} \right)^{\top} + \frac{1}{\eta}\left(\bm{1} - K\pt^{\top}\right) \right\|_{\infty}\left\| y \odot y \right\| + \frac{1}{\eta} \|\bm{1}\|_{\infty}\|y\odot y\|_1 \\ & \qquad+
        2\left\| \left(\fim^{\top}\bm{r} \right)^{\top} + \frac{1}{\eta}\left(\bm{1} - K\pt^{\top}\right)\right\| \|y\|_{2}\|\pt\|_{1}\|y\|_{\infty} \\ & \qquad +\frac{1}{\eta} \left(2\|\bm{1}\|_2 +3K\|\pt\|_2 \right)\|y\|_{2}\|\pt\|_{1}\|y\|_{\infty}\\
        &\leq 3 \left\| \left(\fim^{\top}\bm{r} \right)^{\top} + \frac{1}{\eta}\left(\bm{1} - K\pt^{\top}\right)\right\|_2\|y\|_{2}^{2} + \frac{1}{\eta}\left( 1+2\|\bm{1}\|_2 + 3K\|\pt\|_2 \right)\|y\|_2^2\\
        & \leq 3 \left\| \left(\fim^{\top}\bm{r} \right)^{\top} + \frac{1}{\eta}\left(\bm{1} - K\pt^{\top}\right)\right\|_2\|y\|_{2}^{2} + \frac{3K + 2 \sqrt{K}+1}{\eta}\|y\|_{2}^{2}\\
        & \leq 3 \left\| \left(\fim^{\top}\bm{r} \right)^{\top} + \frac{1}{\eta}\left(\bm{1} - K\pt^{\top}\right)\right\|_2\|y\|_{2}^{2} + \frac{5K}{\eta}\|y\|_{2}^{2}\\
        & \leq 3\left(\normGradPhi + \frac{5K}{\eta} \right)\|y\|_2^2,
    \end{align*}
    where $\bm{1}$ is the $K$-dimensional ones vector and the Cauchy-Schwarz and H{\"o}lder's inequalities have been employed.
\end{proof}

\begin{lemma}[Upper Bound on the Sample Gradient Norm]
    \label{lemma:ub_sample_gradient_norm}
    Given $\bm{r} \in [-\Rmax, \Rmax]$,  $K > 1$, and $\eta > 0$, for all $t\ge 1$, the sample gradient of the barrier function is upper bounded almost surely as
    \begin{align}
        \sampleNormGradPhiT \leq \sqrt{2}\Rmax(1-\pi_{\bt_t}(a_t)) + \frac{2}{\eta}K
    \end{align}
\end{lemma}
\begin{proof}
We have,
    \begin{align}
        \sampleNormGradPhiT &\leq \|\nabla_{\bt}(\pi_{\bt_t}^{\top}\hat{\bm{r}})\|_2 + \frac{1}{\eta}\left\|\bm{1} - K\pi_{\bt_t}^{\top} \right\|_2 \label{proof:sample_gradient_bound_0} \\
        & \leq \sqrt{2}R_{\max}(1-\pi_{\bt_t}(a_t)) + \frac{1}{\eta}\left(\|\bm{1}\| + K \|\pi_{\bt_t}\|\right) \label{proof:sample_gradient_bound_1}\\
        & \leq \sqrt{2}R_{\max}(1-\pi_{\bt_t}(a_t)) + \frac{1}{\eta}\left( \sqrt{K} + K\right)\nonumber\\
        & \leq \sqrt{2}R_{\max}(1-\pi_{\bt_t}(a_t)) + \frac{2}{\eta}K
    \end{align}
where, in line~\eqref{proof:sample_gradient_bound_0}, the triangle inequality has been applied, and the analytic gradient has been substituted, in line~\eqref{proof:sample_gradient_bound_1}, a bound for the sampled gradient's performance derived from line (54) of Proposition 3.1 of \citet{mei2024bandit} has been used.
\end{proof}

Since Lemma~\ref{lemma:non-uniform-smoothness} alone is not sufficient to guarantee convergence, we need to prove that the \lbsgb algorithm is able to control the variation of the objective between consecutive iterates using the learning rate $\alpha$.

\begin{restatable}[Non-uniform Smoothness Between Iterates]{lemma}{NUSmoothnessIterates}
\label{lemma:non_uniform_smoothness_iterates}
Using the \lbsgb algorithm with learning rate $\alpha \in \left(0, \frac{1}{6\left( \sqrt{2}\Rmax + \frac{2}{\eta}K \right)}\right)$, we have, for all $t \geq 1$, almost surely,
\begin{align*}
    &\Big| \barrierTnext - \barrierT - \langle \nabla_{\bt}\barrierT, \bt_{t+1} - \bt_t \rangle \Big| \nonumber \\
    &\leq \left( \frac{3 \normGradPhiT}{2 - 6\alpha \left( \sqrt{2}R_{\max} + \frac{2}{\eta}K \right)} + \frac{15K}{\eta} \right) \|\bt_{t+1} - \bt_t\|_2^2.
\end{align*}
\end{restatable}
\begin{proof}

    Denote $\bt_{\zeta} \coloneqq \bt_t + \zeta (\bt_{t+1} - \bt_t)$, with $\zeta \in [0, 1]$. According to Taylor's theorem, we have,
    \begin{align}
        \Big|\barrierTnext - \barrierT - \Big\langle \nabla_{\bt}\barrierT, \bt_{t+1} - \bt_{t} \Big\rangle \Big| & = \frac{1}{2} \Big|(\bt_{t+1} - \bt_t)^{\top} H(\bt_{\zeta})(\bt_{t+1} - \bt_t)\Big|\nonumber\\
        & \leq \frac{3}{2}\left(\normGradPhiZeta + \frac{5K}{\eta} \right)\|\bt_{t+1} - \bt_t\|_2^2\label{proof:ascent_lemma_0},
    \end{align}
    where Lemma~\ref{lemma:non-uniform-smoothness} has been applied in the last step. Now, we need to control the gradient at $\bt_{\zeta}$ with the gradient at $\bt_t$. Denote $\bt_{\zeta_1} \coloneqq \bt_t + \zeta_{1}(\bt_{\zeta} - \bt_{t})$ with some $\zeta_1 \in [0, 1]$. We have,
    \begin{align}
        \|\nabla_{\bt} \Phi_{\eta}(\bt_{\zeta}) - \nabla_{\bt}\barrierT\|_2 & = \Bigg\| \int_{0}^{1} \Big\langle \nabla^{2}_{\bt}\Phi_\eta(\bt_{\zeta_1}), \bt_{\zeta} - \bt_t \Big\rangle d\zeta_1\Bigg\|_2 \label{proof:ascent_lemma_1}\\
        & \leq \int_{0}^{1} \|\nabla^{2}_{\bt}\Phi_\eta(\bt_{\zeta_1})\|_2 \|\bt_{\zeta} - \bt_t\|_2d\zeta_1\label{proof:ascent_lemma_2}\\
        & \leq 3\int_{0}^{1} \left(\|\nabla_{\bt}\Phi_\eta(\bt_{\zeta_1})\|_2 + \frac{5K}{\eta} \right)\|\bt_{\zeta} - \bt_t\|_2d\zeta_1\label{proof:ascent_lemma_3} \\
        & = 3\int_{0}^{1} \left(\|\nabla_{\bt}\Phi_\eta(\bt_{\zeta_1})\|_2 + \frac{5K}{\eta} \right)\zeta\|\bt_{t+1} - \bt_t\|_2 d\zeta_1\label{proof:ascent_lemma_4}\\
        & \leq 3\int_{0}^{1} \left(\|\nabla_{\bt}\Phi_\eta(\bt_{\zeta_1})\|_2 + \frac{5K}{\eta} \right)\alpha\sampleNormGradPhiT d\zeta_1\label{proof:ascent_lemma_5},
    \end{align}
    where in line~\eqref{proof:ascent_lemma_1} the fundamental theorem of calculus is invoked,  line~\eqref{proof:ascent_lemma_2} applies Cauchy–Schwarz, line~\eqref{proof:ascent_lemma_3} applies Lemma~\ref{lemma:non-uniform-smoothness}, line~\eqref{proof:ascent_lemma_4} substitutes the definition of $\bt_{\zeta}$, and finally, line~\eqref{proof:ascent_lemma_5} substitutes the update rule and uses $\zeta \in [0,1]$. Therefore, we have:
    \begin{align}
        \normGradPhiZeta & \leq \normGradPhiT + \|\nabla_{\bt} \Phi_{\eta}(\bt_{\zeta}) - \nabla_{\bt}\barrierT\|_2 \label{proof:ascent_lemma_6}\\
        & \leq \normGradPhiT + 3\int_{0}^{1} \left(\|\nabla_{\bt}\Phi_\eta(\bt_{\zeta_1})\|_2 + \frac{5K}{\eta} \right)\alpha\sampleNormGradPhiT d\zeta_1 \nonumber\\
        & = \normGradPhiT + \frac{15\alpha K}{\eta} \sampleNormGradPhiT \nonumber\\&\quad + 3\alpha\sampleNormGradPhiT \int_{0}^{1} \|\nabla_{\bt}\Phi_\eta(\bt_{\zeta_1})\|_2 d\zeta_1,\label{proof:ascent_lemma_7}
    \end{align}
    where line~\eqref{proof:ascent_lemma_6} follows from the triangle inequality and line~\eqref{proof:ascent_lemma_7} is obtained applying line~\eqref{proof:ascent_lemma_5}. Denote $\bt_{\zeta_2} \coloneqq \bt_t + \zeta_2 (\bt_{\zeta_1} - \bt_t)$, with $\zeta_{2}\in[0, 1]$. Using similar calculations as in line~\eqref{proof:ascent_lemma_1}, we have,
    \begin{align}
    \|\nabla_{\bt}\Phi_\eta(\bt_{\zeta_1})\|_2 & \leq \normGradPhiT + \|\nabla_{\bt} \Phi_{\eta}(\bt_{\zeta_1}) - \nabla_{\bt}\barrierT\|_2 \nonumber\\
        & \leq \normGradPhiT + 3\int_{0}^{1} \left(\normGradPhiZetaTwo + \frac{5K}{\eta} \right)\alpha\sampleNormGradPhiT d\zeta_2 \nonumber\\
        & = \normGradPhiT + \frac{15\alpha K}{\eta} \sampleNormGradPhiT\nonumber\\&\quad + 3\alpha\sampleNormGradPhiT \int_{0}^{1} \normGradPhiZetaTwo d\zeta_2,\label{proof:ascent_lemma_8}.
    \end{align}
    Combining line~\eqref{proof:ascent_lemma_7} and line~\eqref{proof:ascent_lemma_8}, we have,
    \begin{align*}
        \normGradPhiZeta &\leq \left(1 + 3\alpha \sampleNormGradPhiT\right)\left( \normGradPhiT + \frac{15\alpha K}{\eta} \sampleNormGradPhiT \right) \\
        &\qquad + \left(3\alpha \sampleNormGradPhiT\right)^2\int_{0}^{1}\int_{0}^{1} \normGradPhiZetaTwo d\zeta_2 d\zeta_1,
    \end{align*}
    iterating as $\bt_{\zeta_{i}}=\zeta_{i-1}(\bt_{\zeta_{i-1}}-\bt_t)$, $\zeta_{i}\in[0,1]$, $i=2,\dots,n$, we obtain:
    \begin{equation}
        \begin{aligned}
        \normGradPhiZeta &\leq \left( \normGradPhiT + \frac{15\alpha K}{\eta} \sampleNormGradPhiT \right)\sum_{i=0}^{n-1}\left( 3\alpha \sampleNormGradPhiT \right)^{i} \\ 
        &\qquad+\left(3\alpha \sampleNormGradPhiT\right)^n\int_{0}^{1}\dots\int_{0}^{1} \|\nabla_{\bt}\Phi_\eta(\bt_{\zeta_n})\| d\zeta_n\dots d\zeta_1.
        \end{aligned}\label{proof:ascent_lemma_9}
    \end{equation}

    Now, we select the learning rate $\alpha$ and the barrier parameter $\eta$ such that the argument of the summation is less than 1, so that the sum converges and the remainder vanishes. Following Lemma~\ref{lemma:ub_sample_gradient_norm}, we have to select $\alpha$ such that,
    \begin{align*}
        3\alpha\left( \sqrt{2}\Rmax(1 - \pi_{\bt_t}(a_t)) + \frac{2}{\eta}K \right) \leq 3\alpha\left( \sqrt{2}\Rmax + \frac{2}{\eta}K \right) < 1 \Rightarrow \alpha < \frac{1}{3\left( \sqrt{2}\Rmax + \frac{2}{\eta}K \right)}.
    \end{align*}
    Now taking a limit $n\to\infty$ in Equation \eqref{proof:ascent_lemma_9}, since $3\alpha \sampleNormGradPhiT<1$ almost surely and the argument of the multiple integral is bounded (see Lemma \ref{lemma:UB_gradient_norm}),
    \begin{equation}
        \begin{aligned}
        \normGradPhiZeta &\leq \left( \normGradPhiT + \frac{15\alpha K}{\eta} \sampleNormGradPhiT \right)\sum_{i=0}^{\infty}\left( 3\alpha \sampleNormGradPhiT \right)^{i}.
        \end{aligned}\label{proof:ascent_lemma_9b}
    \end{equation}
    
    As $\eta \rightarrow \infty$, the learning rate $\alpha$ approaches its limit value. We can then select $\alpha$ to be a half its limit value, hence $\alpha \leq \frac{1}{6\left( \sqrt{2}\Rmax + \frac{2}{\eta}K \right)}$,  making the series in Equation~\eqref{proof:ascent_lemma_9b} to converge to $\frac{1}{1-3\alpha \sampleNormGradPhiT}$. Hence, we have,
    \begin{align}
        \normGradPhiZeta & \leq \frac{1}{1-3\alpha \sampleNormGradPhiT}\left( \normGradPhiT + \frac{15\alpha K}{\eta} \sampleNormGradPhiT \right)\nonumber\\
        & \leq \frac{1}{1-3\alpha\left( \sqrt{2}\Rmax + \frac{2}{\eta}K \right)}\left( \normGradPhiT + \frac{15\alpha K}{\eta} \left( \sqrt{2}\Rmax + \frac{2}{\eta}K \right) \right)\nonumber\\
        &\leq \frac{1}{1-3\alpha\left( \sqrt{2}\Rmax + \frac{2}{\eta}K \right)}\normGradPhiT + \frac{5}{\eta}K\frac{3\alpha\left(\sqrt{2}\Rmax + \frac{2}{\eta}K\right)}{1-3\alpha \left(\sqrt{2}\Rmax + \frac{2}{\eta}K\right)} \nonumber\\
        & \leq \frac{1}{1-3\alpha\left( \sqrt{2}\Rmax + \frac{2}{\eta}K \right)}\normGradPhiT + \frac{5}{\eta}K \label{proof:ascent_lemma_10},
    \end{align}
    where line~\eqref{proof:ascent_lemma_10} follows from the fact that $3\alpha\left(\sqrt{2}\Rmax + \frac{2}{\eta}K \right)< \frac{1}{2}$ due to the choice of $\alpha$.
    Finally, combining line~\eqref{proof:ascent_lemma_0} with line~\eqref{proof:ascent_lemma_10}, we have, 
    \begin{align*}
        &\Big|\barrierTnext - \barrierT - \Big\langle \nabla_{\bt}\barrierT, \bt_{t+1} - \bt_{t} \Big\rangle \Big|\\
        &\qquad\qquad \leq \frac{3}{2}\left(\frac{1}{1-3\alpha\left( \sqrt{2}\Rmax + \frac{2}{\eta}K \right)}\normGradPhiT + \frac{10K}{\eta} \right)\|\bt_{t+1} - \bt_t\|_2^2 \\
        &\qquad\qquad=\left(\frac{3}{2-6\alpha\left( \sqrt{2}\Rmax + \frac{2}{\eta}K \right)}\normGradPhiT + \frac{15K}{\eta} \right)\|\bt_{t+1} - \bt_t\|_2^2.
    \end{align*}
\end{proof}

\Lojasiewicz*
\begin{proof}
    Using the definition of the Jacobian of the barrier function, we have,
    \begin{align*}
        \|\nabla_{\bt} \barrier\|_2  & = \Big\| (\fim^{\top}\bm{r}) + \frac{1}{\eta} \left( \bm{1} - K\pt^{\top} \right) \Big\|_2\\
        & = \left( \sum_{a \in \dsb{K}}\left(\pt(a)(r(a) - \pt^{\top}\bm{r}) + \frac{1}{\eta}(1 - K\pt(a)) \right)^2\right)^{\frac{1}{2}}\\
        & \geq \left(\pt(a^*)(r(a^*) - \pt^{\top}\bm{r}) - \frac{1}{\eta}\left| 1 -K\pt(a^*) \right|\right)^+\\
        & \geq \left( \pt(a^*)(r(a^*) - \pt^{\top}\bm{r}) - \frac{1}{\eta}(K-1)\right)^+.
    \end{align*}
\end{proof}

\subsection{Proofs for Section~\ref{subsec:bounded_c_star}}
Here, we provide the proofs for the Lemmas and Theorems provided in Section~\ref{subsec:bounded_c_star}. We start showing the self-bounding property of the barrier function in Lemma~\ref{lemma:strong_growth}. Then, after analyzing the difference of the log-barrier between iterates in Lemma~\ref{lemma:barrier_difference}, we show why \lbsgb converges with constant learning rates in Lemma~\ref{lemma:constant_learning_rates}. Finally, after providing an upper bound on the true gradient norm in Lemma~\ref{lemma:UB_gradient_norm}, we provide the convergence rate of our algorithm in Theorem~\ref{thr:convergence_rate}.

\strongGrowth*
\begin{proof}
    Considering the expectation of $\sampleNormGradPhiT^2$, we have,
    \begin{align}
        \mE_{t}\left[ \sampleNormGradPhiT^2 \right] &= \mE_{t}\left[ \Big\| \nabla_{\bt}\big(\pi_{\bt_t}^{\top}\hat{\bm{r}}) + \frac{1}{\eta}\nabla_{\bt}\sum_{a\in\dsb{K}}\log \pi_{\bt_t}(a) \Big\|_2^2 \right]\nonumber\\
        & \leq  2\mE_{t}\left[ \Big\| \nabla_{\bt}\big(\pi_{\bt_t}^{\top}\hat{\bm{r}})\Big\|_2^2\right] + 2\mE_{t}\left[ \Big\| \frac{1}{\eta}\nabla_{\bt}\sum_{a\in\dsb{K}}\log \pi_{\bt_t}(a)\Big\|_2^2\right]\label{proof:self_bounding_2}\\
        & \leq \frac{16 \Rmax^3 K^{3/2}}{\Delta^2} \|\nabla_{\bt} (\pi_{\theta_t}^{\top}\bm{r}) \pm \frac{1}{\eta} \nabla_{\bt} \sum_{a\in \dsb{K}} \log \pi_{\bt_t}(a)\|_2 + 8\left(\frac{1}{\eta}K\right)^2\label{proof:self_bounding_3}\\
        & \leq  \frac{16 \Rmax^3 K^{3/2}}{\Delta^2} \normGradPhiT + \frac{2}{\eta}K\left( \frac{4}{\eta}K +  \frac{16 \Rmax^3 K^{3/2}}{\Delta^2} \right)\nonumber,
    \end{align}
    where %
    line~\eqref{proof:self_bounding_2} follows from Young's inequality, and line~\eqref{proof:self_bounding_3} is obtained applying Lemma 4.3 from \cite{mei2021leveraging} and bounding the deterministic term similarly as Lemma~\ref{lemma:ub_sample_gradient_norm}.
\end{proof}

\begin{lemma}[Difference Between Barriers Among Iterates]
\label{lemma:barrier_difference}
Given the learning rate $\alpha>0$ and $K$ arms, considering the update $\bt' \leftarrow \bt + \alpha \widehat{\nabla}_{\bt}\barrier$, the difference between the barrier terms among two consecutive iterates is almost surely bounded as
\begin{align}
    \sum_{a\in \dsb{K}}\left(\log\pi_{\bt'}(a) - \log\pi_{\bt}(a)\right)\leq 2\alpha K \sampleNormGradPhi
        \leq 2\alpha K \left( \sqrt{2}\Rmax + \frac{2}{\eta}K \right)
\end{align}
\end{lemma}
\begin{proof}
    \begin{align}
        \sum_{a\in \dsb{K}}\left(\log\pi_{\bt'}(a) - \log\pi_{\bt}(a)\right) & = \sum_{a\in \dsb{K}} \log \frac{\pi_{\bt'}(a)}{\pi_{\bt}(a)}\nonumber \\ 
        & = \sum_{a\in \dsb{K}} \log \frac{\frac{e^{\theta'(a)}}{\sum_b e^{\theta'(b)}}}{\frac{e^{\theta(a)}}{\sum_b e^{\theta(b)}}}\nonumber\\ 
        & = \sum_{a\in \dsb{K}} \log \frac{e^{\theta'(a)}}{e^{\theta(a)}} \frac{\sum_b e^{\theta(b)}}{\sum_b e^{\theta'(b)}}\nonumber\\ 
        & = \sum_{a\in \dsb{K}}(\theta'(a) - \theta(a))+K\log \frac{\sum_b e^{\theta(b)}}{\sum_b e^{\theta'(b)}} \nonumber\\
        & = \alpha \sum_{a\in \dsb{K}}\frac{\hat{\partial}}{\partial \theta(a)}\Phi_{\eta}(\theta(a)) \nonumber\\&\quad+K\left(\log\left(\sum_{b}e^{\theta(b)}\right)-\log\left(\sum_{b}e^{\theta'(b)}\right)\right)\nonumber\\
        &\leq \alpha \left\langle\widehat{\nabla}_{\bt}\barrier, \bm{1} \right\rangle +K \left\langle\bt - \bt', \pi(\bt) \right\rangle\label{proof:barrier_difference_1}\\
        & = \alpha \left\langle \widehat{\nabla}_{\bt}\barrier, \bm{1} \right\rangle -K \left\langle\alpha \widehat{\nabla}_{\bt} \Phi_{\eta}(\bt), \pi(\bt) \right\rangle\nonumber\\
        & \leq \alpha \sampleNormGradPhi \|\bm{1}\|_2 + \alpha K \sampleNormGradPhi \|\pt\|_2\nonumber\\
        & \leq \alpha \sqrt{K} \sampleNormGradPhi + K\sampleNormGradPhi \nonumber\\
        & \leq 2\alpha K\sampleNormGradPhi,\nonumber\\
    \end{align}
    where line~\eqref{proof:barrier_difference_1} follows from the convexity of the LogSumExp, and $\pi(\bt)$ is the respective policy. Finally, applying Lemma~\ref{lemma:ub_sample_gradient_norm},
    \begin{align*}
        \sum_{a\in \dsb{K}}\left(\log\pi_{\bt'}(a) - \log\pi_{\bt}(a)\right) & \leq 2\alpha K \sampleNormGradPhi\\
        \leq 2\alpha K \left( \sqrt{2}\Rmax + \frac{2}{\eta}K \right)
    \end{align*}
\end{proof}

In order to show the global convergence of the \lbsgb with constant learning rate, we need to provide a descent lemma leveraging the non-smoothness and self-bounding properties.

\begin{lemma}[Descent Lemma]
    \label{lemma:descent_lemma}
    We have, for all $t \geq 1$, almost surely,
    \begin{align}
        \pi_{t}^{\top}\bm{r} - \mE_{t}[\pi_{t+1}^{\top}\bm{r}] &\leq \left(\alpha^2 \frac{80}{3}\frac{\Rmax^3 K^{3/2}}{\Delta^2}-\alpha\right)\normGradPhiT^2 + \beta\left( \bt_t, \frac{1}{\eta}\right)
\end{align}
    where the bias term $\beta$ is defined as,
    \begin{align*}
        \beta\left( \bt_t, \frac{1}{\eta}\right) \leq \frac{\alpha}{\eta}K\left( 2\normGradPhiT + 2\sqrt{2}\Rmax + \frac{5K}{\eta} \right)
    \end{align*}
\end{lemma}
\begin{proof}
    According to Lemma~\ref{lemma:non_uniform_smoothness_iterates}, we have,
    \begin{align}
        &\Big|\barrierTnext - \barrierT - \Big\langle \nabla_{\bt}\barrierT, \bt_{t+1} - \bt_{t} \Big\rangle \Big|  \nonumber\\
        &\qquad\qquad\leq\left(\frac{3}{2 - 6\alpha \left( \sqrt{2}\Rmax + \frac{2}{\eta}K \right)}\normGradPhiT + \frac{15}{\eta}K\right) \|\bt_{t+1} - \bt_{t}\|_2^2\nonumber\\
        &\qquad\qquad \leq\left( \frac{5}{3}\normGradPhiT + \frac{15}{\eta}K \right)\|\bt_{t+1} - \bt_{t}\|_2^2\label{proof:constant_lr_3},
    \end{align}
    where line~\eqref{proof:constant_lr_3} follows from the bound on learning rate $\alpha$ in line~\eqref{proof:constant_lr_2}.
    We have,
    \begin{align}
        \barrierT - \barrierTnext = (\pi_{\bt_{t}}^{\top}\bm{r}) - (\pi_{\bt_{t+1}}^{\top}\bm{r}) - \frac{1}{\eta}\sum_{a\in \dsb{K}}\left(\log\pi_{\bt_{t+1}}(a) - \log\pi_{\bt_{t}}(a)\right),
    \end{align}
    and, applying Lemma~\ref{lemma:barrier_difference} with $b\left( \bt_t, \frac{1}{\eta}\right) \coloneqq \frac{2\alpha K}{\eta} \sampleNormGradPhi$
    and applying Lemma~\ref{lemma:non_uniform_smoothness_iterates} and line~\eqref{proof:constant_lr_3}, we have,
    \begin{align}
         (\pi_{\bt_{t}}^{\top}\bm{r})-(\pi_{\bt_{t+1}}^{\top}\bm{r}) & \leq -\alpha \left\langle \nabla\Phi_\eta(\bt_t), \bt_{t+1} - \bt_{t} \right\rangle \nonumber\\&\quad+ \left( \frac{5}{3}\normGradPhiT + \frac{15}{\eta}K \right)\|\bt_{t+1} - \bt_{t}\|_2^2 + b\left( \bt_t, \frac{1}{\eta}\right)\nonumber\\
         &\leq -\alpha \left\langle \nabla\Phi_\eta(\bt_t), \widehat{\nabla}\Phi_\eta(\bt_t) \right\rangle \nonumber\\
         &\quad + \left( \frac{5}{3}\normGradPhiT + \frac{15}{\eta}K \right)\alpha^2 \sampleNormGradPhiT^2 + b\left( \bt_t, \frac{1}{\eta}\right)\label{proof:constant_lr_4},
    \end{align}
    where in line~\eqref{proof:constant_lr_4} the update rule has been used. Taking the expectation, we have,
    \begin{align}
         (\pi_{\bt_{t}}^{\top}\bm{r})-\mE_t\left[(\pi_{\bt_{t+1}}^{\top}\bm{r})\right] & \leq -\alpha \left\langle \normGradPhiT, \mE\left[\widehat{\nabla}\Phi_\eta(\bt_t)\right] \right\rangle\nonumber\\
         &\quad + \left( \frac{5}{3}\normGradPhiT + \frac{15}{\eta}K \right)\alpha^2 \mE_{t}\left[\sampleNormGradPhiT^2\right] + \mE_{t}\left[b\left( \bt_t, \frac{1}{\eta}\right)\right]\nonumber\\
         & =-\alpha \normGradPhiT^2+ \left( \frac{5}{3}\normGradPhiT + \frac{15}{\eta}K \right)\alpha^2 \mE_{t}\left[\sampleNormGradPhiT^2\right]\nonumber\\&\quad + \mE_{t}\left[b\left( \bt_t, \frac{1}{\eta}\right)\right]\label{proof:constant_lr_5},
    \end{align}
    where line~\eqref{proof:constant_lr_5} follows from the unbiasedness of the gradient estimator of the barrier function. Let's denote $A = \nabla_{\bt}\Phi_\eta(\bt_t)$, $\hat{A} = \widehat{\nabla}_{\bt}\Phi_\eta(\bt_t)$, $c = \frac{15}{\eta}K$, and $d = \frac{2}{\eta}K\left( \frac{4}{\eta}K +  \frac{16 \Rmax^3 K^{3/2}}{\Delta^2} \right)$. Now, we have,
    \begin{align}
         (\pi_{\bt_{t}}^{\top}\bm{r})-\mE_t\left[(\pi_{\bt_{t+1}}^{\top}\bm{r})\right] & \leq -\alpha \|A\|_2^2 + \left(\frac{5}{3}\|A\|_2 + c\right)\alpha^{2}\mE_{t}[\|\hat{A}\|_2^2] + \mE_{t}\left[b\left( \bt_t, \frac{1}{\eta}\right)\right]\nonumber\\
         &\leq -\alpha \|A\|_2^2 + \left(\frac{5}{3}\|A\|_2 + c\right)\alpha^{2}\left( \frac{16 \Rmax^3 K^{3/2}}{\Delta^2} \|A\|_2 + d \right)\nonumber\\&\quad + \mE_{t}\left[b\left( \bt_t, \frac{1}{\eta}\right)\right]\label{proof:constant_lr_6}\\
         & = -\alpha \|A\|_2^2 + \alpha^2 \frac{80}{3}\frac{\Rmax^3 K^{3/2}}{\Delta^2}\|A\|_2^2 + \alpha^2\frac{5}{3}\|A\|_2\cdot d \\
         &\quad + \alpha^2 \frac{16 \Rmax^3 K^{3/2}}{\Delta^2}\cdot c \cdot \|A\|_2\nonumber+ \alpha^2\cdot c\cdot d +\mE_{t}\left[b\left( \bt_t, \frac{1}{\eta}\right)\right]\nonumber\\
         & = -\alpha \|A\|_2^2 + \alpha^2 \frac{80}{3}\frac{\Rmax^3 K^{3/2}}{\Delta^2}\|A\|_2^2 + \beta\left( \bt_t, \frac{1}{\eta}\right)\nonumber\\
         & = \left(-\alpha + \alpha^2 \frac{80}{3}\frac{\Rmax^3 K^{3/2}}{\Delta^2}\right)\|A\|_2^2+ \beta\left( \bt_t, \frac{1}{\eta}\right),\label{eq:constant_lr_proof}
    \end{align}
         where line~\eqref{proof:constant_lr_6} follows from the application of Lemma~\ref{lemma:strong_growth} We now have to characterize the bias term $\beta$. Let's analyze each component. The term $\alpha^2\frac{5}{3}\|A\|_2 \cdot d$ is $\mathcal{O}(\eta^{-1})$ by definition of $d$, hence it can be upper-bounded by,
         \begin{align*}
             \alpha^2\frac{5}{3}\|A\|_2 \cdot d\leq \frac{\alpha}{\eta}K\|A\|_2.
         \end{align*}
         The term $\alpha^2 \frac{16 \Rmax^3 K^{3/2}}{\Delta^2}\cdot c$ is $\mathcal{O}(\eta^{-1})$ and being $\alpha^2 \frac{16 \Rmax^3 K^{3/2}}{\Delta^2}<1$ we can upper-bound it as,
         \begin{align*}
             \alpha^2 \frac{16 \Rmax^3 K^{3/2}}{\Delta^2}\cdot c\cdot\|A\|_2 \leq \frac{\alpha}{\eta}K\|A\|_2.
         \end{align*}
         The term $\alpha^2\cdot c \cdot d$ is $\mathcal{O}(\eta^{-2})$ and can be upper-bounded as,
         \begin{align*}
             \alpha^2\cdot c \cdot d \leq \frac{\alpha}{\eta^2}K^2,
         \end{align*}
         when $\alpha \leq \frac{\Delta^2}{16\Rmax^3K^{3/2}}$.
         Finally, it is necessary to bound the term $\mE_{t}\left[b\left( \bt_t, \frac{1}{\eta}\right)\right]$. Applying the upper bound from Lemma~\ref{lemma:ub_sample_gradient_norm}, we have almost surely:
         \begin{align}
            b\left( \bt_t, \frac{1}{\eta}\right) \leq \frac{2\alpha K}{\eta} \left( \sqrt{2}\Rmax + \frac{2}{\eta}K \right).
         \end{align}
         Taking the expectation yields:
         \begin{align}
            \mE_{t}\left[b\left( \bt_t, \frac{1}{\eta}\right)\right]
            & \leq \frac{2\alpha K}{\eta} \left( \sqrt{2}\Rmax + \frac{2}{\eta}K \right).
        \end{align}
        Finally, putting all the bias terms together, we obtain,
        \begin{align*}
             \beta\left( \bt_t, \frac{1}{\eta}\right) &\leq \frac{\alpha}{\eta}K\normGradPhiT + \frac{\alpha}{\eta}K\normGradPhiT + \frac{\alpha}{\eta^2}K^2 + \frac{2\alpha K}{\eta} \left( \sqrt{2}\Rmax + \frac{2}{\eta}K \right)\\
             & = \frac{\alpha}{\eta}K\left( 2\normGradPhiT + 2\sqrt{2}\Rmax + \frac{5K}{\eta} \right).
        \end{align*}
\end{proof}

In the following lemma, we show that, with a proper choice of the learning rate $\alpha$, we can prove that applying the \lbsgb update rule, we are progressing toward the optimal policy, up to a constant bias term, which is partly dependent on the true gradient norm of the objective function, and controllable by choosing the barrier parameter $\eta$. Furthermore, Lemma~\ref{lemma:constant_learning_rates} provides a mechanism to choose the learning rate $\alpha$ coherent with the requirement from Lemma~\ref{lemma:non_uniform_smoothness_iterates}.
\begin{lemma}[Constant Learning Rates]
    \label{lemma:constant_learning_rates}
    Selecting the learning rate $\alpha = \frac{3\Delta^2}{160K^{3/2}\Rmax^2(\sqrt{2}\Rmax + \frac{2}{\eta}K)}$, we have, for all $t \geq 1$, almost surely,
    \begin{align}
        \pi_{t}^{\top}\bm{r} - \mE_{t}[\pi_{t+1}^{\top}] \leq -\frac{\alpha}{2}\normGradPhiT^2 + \beta\left( \bt_t, \frac{1}{\eta}\right),
    \end{align}
    where the bias term $\beta$ is defined as,
    \begin{align*}
        \beta\left( \bt_t, \frac{1}{\eta}\right) \leq \frac{\alpha}{\eta}K\left( 2\normGradPhiT + 2\sqrt{2}\Rmax + \frac{5K}{\eta} \right)
    \end{align*}
\end{lemma}
\begin{proof}
    Using the learning rate,
    \begin{align}
        \alpha &= \frac{3\Delta^2}{160 K^{3/2}\Rmax^2\left( \sqrt{2}\Rmax + \frac{2}{\eta}K\right)}\nonumber\\
        &= \frac{3}{160\left( \sqrt{2}\Rmax + \frac{2}{\eta}K\right)}\cdot \frac{\Delta^2}{\Rmax^2}\cdot \frac{1}{K^{3/2}}\nonumber\\
        &\leq \frac{3}{160\left( \sqrt{2}\Rmax + \frac{2}{\eta}K\right)} \cdot 4 \cdot \frac{1}{2\sqrt{2}}\label{proof:constant_lr_1}\\
        &\leq \frac{3\sqrt{2}}{160\left( \sqrt{2}\Rmax + \frac{2}{\eta}K\right)}\label{proof:constant_lr_2},
    \end{align}
    where line~\eqref{proof:constant_lr_1} follows from $\Delta \leq 2\Rmax$ and $K > 2$, we have that $\alpha \in \left( 0, \frac{1}{6\left( \sqrt{2}\Rmax + \frac{2}{\eta}K \right)}\right)$. Now, plugging the selected learning rate $\alpha$ in the result of Lemma~\ref{lemma:descent_lemma}, we have,
    \begin{align*}
        \pi_{t}^{\top}\bm{r} - \mE_{t}[\pi_{t+1}^{\top}] &\leq \left(\alpha^2 \frac{80}{3}\frac{\Rmax^3 K^{3/2}}{\Delta^2}-\alpha\right)\normGradPhiT^2 + \beta\left( \bt_t, \frac{1}{\eta}\right)\\
        &\leq -\frac{\alpha}{2}\normGradPhiT^2 + \beta\left( \bt_t, \frac{1}{\eta}\right)
    \end{align*}
\end{proof}

\begin{lemma}[Upper Bound on Gradient Norm]
\label{lemma:UB_gradient_norm}
For all $t\ge 1$, almost surely, the norm of the gradient function $\normGradPhiT$ is upper bounded as,
\begin{align}
    \normGradPhiT \leq 2\left(r(a^*) - \pi_{\bt_t}^{\top}\bm{r}\right) + \frac{2}{\eta}K,
\end{align}
where $r(a^*)$ is the reward of the optimal arm.
\end{lemma}
\begin{proof}
    Let $\nabla_{\bt}B(\bt_t) = \nabla_{\bt}\frac{1}{\eta}\sum_{a\in \dsb{K}}\log\pi_{\bt_t}(a)$. We have,
    \begin{align*}
        \normGradPhiT &= \left\| \nabla_{\bt}\left( \pi_{\bt_t}^{\top}\bm{r} \right) + \nabla_{\bt}B(\bt_t)\right\|_2\\
        &\leq \|\nabla_{\bt}\left( \pi_{\bt_t}^{\top}\bm{r} \right)\|_2 + \|\nabla_{\bt}B(\bt_t)\|_2\\
        &\leq \|\nabla_{\bt}\left( \pi_{\bt_t}^{\top}\bm{r} \right)\|_2 + \frac{2}{\eta}K,
    \end{align*}
    where the triangle inequality and similar calculations as Lemma~\ref{lemma:ub_sample_gradient_norm} have been used. Focusing only on the first term, we can upper-bound its $L_2$ norm using the $L_1$ norm. We have,
    \begin{align*}
        \|\nabla_{\bt}\left( \pi_{\bt_t}^{\top}\bm{r} \right)\|_2 &\leq \|\nabla_{\bt}\left( \pi_{\bt_t}^{\top}\bm{r} \right)\|_1 \\
        &= \sum_{a \in \dsb{K}} \left| \pi_{\bt_t}(a)\left(r(a)-\pi_{\bt_t}^{\top}\bm{r}\right) \right|\\
        &= \sum_{a \in \dsb{K}} \pi_{\bt_t}(a) \left| r(a)-\pi_{\bt_t}^{\top}\bm{r} \right|.
    \end{align*}
    Let $S^+ = \{a \in \dsb{K} \mid r(a) \ge \pi_{\bt_t}^{\top}\bm{r}\}$ be the set of actions with rewards above the expected value. Because the expected deviation is strictly zero, i.e., $\sum_{a \in \dsb{K}} \pi_{\bt_t}(a)(r(a) - \pi_{\bt_t}^{\top}\bm{r}) = 0$, the sum of the absolute deviations is exactly twice the sum of the positive deviations. Therefore,
    \begin{align*}
        \sum_{a \in \dsb{K}} \pi_{\bt_t}(a) \left| r(a)-\pi_{\bt_t}^{\top}\bm{r} \right| &= \sum_{a \in S^+} \pi_{\bt_t}(a) \left( r(a)-\pi_{\bt_t}^{\top}\bm{r} \right) -\sum_{a \in S^-} \pi_{\bt_t}(a) \left( r(a)-\pi_{\bt_t}^{\top}\bm{r} \right)\\
        &= 2 \sum_{a \in S^+} \pi_{\bt_t}(a) \left( r(a)-\pi_{\bt_t}^{\top}\bm{r} \right)\\
        &\leq 2 \sum_{a \in S^+} \pi_{\bt_t}(a) \left( r(a^*)-\pi_{\bt_t}^{\top}\bm{r} \right)\\
        &= 2 \left( r(a^*)-\pi_{\bt_t}^{\top}\bm{r} \right) \sum_{a \in S^+} \pi_{\bt_t}(a)\\
        &\leq 2 \left( r(a^*)-\pi_{\bt_t}^{\top}\bm{r} \right),
    \end{align*}
    where we used the fact that, $\sum_{a \in S^+} \pi_{\bt_t}(a) \left( r(a)-\pi_{\bt_t}^{\top}\bm{r} \right) + \sum_{a \in S^-} \pi_{\bt_t}(a) \left( r(a)-\pi_{\bt_t}^{\top}\bm{r} \right) = 0$, $r(a) \le r(a^*)$ for all $a$, and $\sum_{a \in S^+} \pi_{\bt_t}(a) \le 1$. Hence, substituting this back into the original triangle inequality, we have,
    \begin{align*}
        \normGradPhiT \leq 2\left(r(a^*) - \pi_{\bt_t}^{\top}\bm{r}\right) + \frac{2}{\eta}K.
    \end{align*}
\end{proof}

\ConvergenceRate*
\begin{proof}
    Now, recalling that, from the proof of Lemma~\ref{lemma:descent_lemma},
    \begin{align*}
         \beta\left( \bt_t, \frac{1}{\eta}\right) &\leq \frac{\alpha}{\eta}K\left( 2\normGradPhiT + 2\sqrt{2}\Rmax + \frac{5K}{\eta} \right)\nonumber\\ 
         &\leq \frac{\alpha}{\eta}K\left(4r(a^*) - 4\pi_{\bt_t}^{\top}\bm{r} + 2\sqrt{2}\Rmax + \frac{9K}{\eta}\right)
     \end{align*}
     where we applied Lemma~\ref{lemma:UB_gradient_norm} for the second inequality. Since we want to progress toward the optimal policy, we need the first term of the right-hand side to be negative. Letting $\alpha = \frac{3\Delta^2}{160 K^{3/2}\Rmax^2\left( \sqrt{2}\Rmax + \frac{2}{\eta}K\right)}$, similarly to Lemma~\ref{lemma:constant_learning_rates}, we have,
    \begin{align*}
        \left(\alpha^2 \frac{80}{3}\frac{\Rmax^3 K^{3/2}}{\Delta^2}-\alpha\right) \leq -\frac{\alpha}{2},
    \end{align*}
    which implies,
    \begin{align}
        \pi_{t}^{\top}\bm{r} - \mE_{t}[\pi_{t+1}^{\top}] &\leq -\frac{\alpha}{2}\normGradPhiT^2 + \beta\left( \bt_t, \frac{1}{\eta}\right)\nonumber\\
        & \leq -\frac{\alpha}{2}\left(\left[\pi_{\bt_t}(a^*)\left(r(a^*) - \pi_{\bt_t}^{\top}\bm{r}\right) - \frac{K-1}{\eta}\right]_+\right)^2 +\beta\left( \bt_t, \frac{1}{\eta}\right)\label{proof:convergence_1_orig}\\
        & \leq -\frac{\alpha}{2}\pi_{\bt_t}(a^*)^{2}\left(r(a^*) - \pi_{\bt_t}^{\top}\bm{r}\right)^2 + \frac{\alpha}{\eta}\pi_{\bt_t}(a^*)\left(r(a^*) - \pi_{\bt_t}^{\top}\bm{r}\right)(K-1) \nonumber\\
        &\quad +\beta\left( \bt_t, \frac{1}{\eta}\right)\label{proof:convergence_1}
    \end{align}
    where line~\eqref{proof:convergence_1_orig} follows from Lemma~\ref{lemma:NL}, and line~\eqref{proof:convergence_1} applies Lemma~\ref{lemma:positive_part_bound} to strictly lower-bound the squared positive part operator. Substituting the upper bound for $\beta\left(\bt_t, \frac{1}{\eta}\right)$, we have,
    \begin{align}
         \pi_{t}^{\top}\bm{r} -\mE_{t}[\pi_{t+1}^{\top}] & \leq -\frac{\alpha}{2}\cdot\pi_{\bt_t}(a^*)^{2}\left( r(a^*) - \pi_{\bt_t}^{\top}\bm{r} \right)^2 \nonumber\\
        &\quad+\frac{\alpha}{\eta}\cdot\pi_{\bt_t}(a^*)\left( r(a^*) - \pi_{\bt_t}^{\top}\bm{r} \right)(K-1)\nonumber\\
        &\quad +\frac{\alpha}{\eta}K\left(4r(a^*) - 4\pi_{\bt_t}^{\top}\bm{r} + 2\sqrt{2}\Rmax + \frac{9K}{\eta}\right).
    \end{align}
    Rearranging the terms, and letting $\alpha\left(\pi_{\bt_t}(a^*)(K-1)+4K\right) \leq \alpha\left((K-1)+4K\right) = \alpha(5K-1) \coloneqq B$ and $C \coloneqq \alpha K \left(2\sqrt{2}\Rmax + \frac{9K}{\eta}\right)$, we have,
    \begin{align*}
        \pi_{t}^{\top}\bm{r} - \mE_{t}[\pi_{t+1}^{\top}] &\leq -\frac{\alpha}{2}\cdot\pi_{\bt_t}(a^*)^{2}\left(r(a^*) - \pi_{\bt_t}^{\top}\bm{r}\right)^2 + \frac{1}{\eta}B \left(r(a^*) - \pi_{\bt_t}^{\top}\bm{r}\right) + \frac{1}{\eta}C\\
        & = -\frac{\alpha}{2}\cdot\pi_{\bt_t}(a^*)^{2}\left(r(a^*) - \pi_{\bt_t}^{\top}\bm{r}\right)^2 + \frac{1}{\eta}B \left(r(a^*) - \pi_{\bt_t}^{\top}\bm{r}\right) + \frac{1}{\eta}C\\
        & \leq -\frac{\alpha}{2}\pi_{\bt_t}(a^*)^{2}\left(r(a^*) - \pi_{\bt_t}^{\top}\bm{r}\right)^2 + \frac{1}{\eta}B \left(r(a^*) - \pi_{\bt_t}^{\top}\bm{r}\right) + \frac{1}{\eta}C,
    \end{align*}
     with the terms $\alpha, B, C$ being independent on the iteration $t$. Denoting the sub-optimality gap as $\delta(\bt_t)\coloneqq (\pi^* - \pi_{\bt_t})^{\top}\bm{r}$, we have,
    \begin{align*}
         \mE_t\left[\delta(\bt_{t+1})\right] - \delta(\bt_t) & =   \mE\left[(\pi^* - \pi_{\bt_{t+1}})^{\top}\bm{r}\right]-(\pi^* - \pi_{\bt_t})^{\top}\bm{r}\\
        & = \pi_{\bt_t}^{\top}\bm{r} - \mE_{t}[\pi_{\bt_{t+1}}^{\top}]\\
        & \leq -\frac{\alpha}{2}\cdot\pi_{\bt_t}(a^*)^{2}\cdot\delta(\bt_t)^2 + \frac{1}{\eta}B \cdot\delta(\bt_t) + \frac{1}{\eta}C.
    \end{align*}
    Taking the expectation, since the previous inequality holds for all $t\ge1$ almost surely and all the involved quantities are uniformly bounded, we have,
    \begin{align}
         \mE\left[\delta(\bt_{t+1})\right] - \mE\left[ \delta(\bt_t) \right] & \leq -\frac{\alpha}{2}\mE\left[\pi_{\bt_t}(a^*)^{2}\cdot\delta(\bt_t)^2\right] + \mE\left[\frac{1}{\eta}B \cdot\delta(\bt_t)\right] + \frac{1}{\eta}C\label{eq:complete_recurrence}\\
        &\leq -\frac{\alpha}{2} \frac{\mE\left[\delta(\bt_t)\right]^2}{\mE\left[\frac{1}{\pi_{\bt_t}(a^*)^2}\right]} + \mE\left[\frac{1}{\eta}B \cdot\delta(\bt_t)\right] + \frac{1}{\eta}C\label{proof:convergence_3}\\
        &\leq -\frac{\alpha}{2c_t^*} \mE\left[\delta(\bt_t)\right]^2 + \frac{1}{\eta}B\cdot \mE\left[\delta(\bt_t)\right] + \frac{1}{\eta}C\nonumber
    \end{align}
    where line~\eqref{proof:convergence_3} follows from \citet{baudrydoes} Appendix B.1, and $c_t^* \coloneqq \mE\left[\frac{1}{\pi_{\bt_t}(a^*)^2}\right]$. Letting $r_t \coloneqq \mE\left[ \delta(\bt_t) \right]$, and noting that $B \leq \alpha(5K-1) \coloneqq \alpha \tilde{B}$ since $\pi_{\bt_t}(a^*)\leq 1$, and letting $C = \alpha K \left(2\sqrt{2}\Rmax + \frac{9K}{\eta}\right) \coloneqq \alpha \tilde{C}$, we can then rewrite the previous inequality as follows:
    \begin{align*}
        r_{t+1} &\leq r_t - \frac{\alpha}{2c_t^*} r_t^2 + \alpha\frac{1}{\eta}\tilde{B}\cdot r_t + \alpha \frac{1}{\eta} \tilde{C} \\
        &\leq r_t - \frac{\alpha}{2c_t^*} r_t^2 + \alpha\left(\frac{1}{\eta}\tilde{B}\cdot 2\Rmax  +  \frac{1}{\eta}\tilde{C}\right)\\
        & \leq r_t - \frac{\alpha}{2\sup_{t\geq 0} c_t^*} r_t^2 + \frac{\alpha}{\eta}\left(\tilde{B}\cdot 2\Rmax  +  \tilde{C}\right)
    \end{align*}
    Where the second inequality holds since $r_t \leq 2\Rmax$. Letting $b \coloneqq \left(\tilde{B}\cdot 2\Rmax  +  \tilde{C}\right)$ and $c^* \coloneqq \sup_{t\geq 0} c_t^*$, following the analysis of the recurrence as done in Section~\ref{subsec:recurrences}, and choosing the learning rate $\alpha$ as,
    \begin{align}
        \alpha = \min \left\{\frac{3\Delta^2}{160K^{3/2}\Rmax^2(\sqrt{2}\Rmax + \frac{2}{\eta}K)}, \frac{c^*}{r_0}, \sqrt{\frac{\eta c^*}{2b}} \right\}. \label{eq:lr}
    \end{align}
    We can derive the convergence rate setting $T \leftarrow t+1$:
    \begin{align*}
        \eDeltaT &\leq \left(1-\frac{1}{2}\sqrt{\frac{2\alpha^2  b}{\eta c^*}}\right)^{T}\eDeltaZero + \sqrt{\frac{2b c^*}{\eta}}\\
        &\leq C^{T}\eDeltaZero + \beta
    \end{align*}
    
    To find the iteration complexity, let $\delta_0 \coloneqq \eDeltaZero$. First, we need to choose $\eta$ to achieve $\eDeltaT \leq \epsilon$ assuming $c^* \geq 0$ is small and known. We have,
    \begin{align}
        \sqrt{\frac{2b c^*}{\eta}} \leq \frac{\epsilon}{2} \Rightarrow \eta \geq 8bc^* \epsilon^{-2}\label{eq:eta_choice}
    \end{align}
    Then, we have,
    \begin{align}
        \left(1-\frac{1}{2}\sqrt{\frac{2\alpha^2  b}{\eta c^*}}\right)^{T}\eDeltaZero \leq \frac{\epsilon}{2} \Rightarrow T \geq \frac{\log \frac{2\delta_0}{\epsilon}}{\log \frac{1}{1-\frac{1}{2}\sqrt{\frac{2\alpha^2  b}{\eta c^*}}}} \geq \sqrt{\frac{2 \eta c^*}{\alpha^2 b}} \log \frac{2\delta_0}{\epsilon},
    \end{align}
    where the fact that $\log\frac{1}{1-x} \geq x$ has been used. Now, substituting line~\eqref{eq:eta_choice} in the iteration complexity, we have,
    \begin{align}
        T \geq \sqrt{\frac{2 \eta c^*}{\alpha^2 b}} \log \frac{2\delta_0}{\epsilon} = \frac{4c^*\epsilon^{-1}}{\alpha}\log \frac{2\delta_0}{\epsilon}. \label{eq:iteration_complexity}
    \end{align}
    Now, we need to substitute in line~\eqref{eq:iteration_complexity} the constraints on the learning rate $\alpha$ in line~\eqref{eq:lr} to get the iteration complexity.
    
    \paragraph{Case 1.} Let $\alpha = \frac{3\Delta^2}{160K^{3/2}\Rmax^2(\sqrt{2}\Rmax + \frac{2}{\eta}K)}$. Substituting the value of $\eta$ with its minimum requirement, hence, $\eta = 8bc^* \epsilon^{-2}$ from line~\eqref{eq:eta_choice}, we have:
    \begin{align*}
        \alpha &= \frac{3\Delta^2}{160K^{3/2}\Rmax^2\left(\sqrt{2}\Rmax + \frac{K\epsilon^2}{4bc^*}\right)}
    \end{align*}
    
    In the asymptotic regime where $\epsilon \to 0$, the term $\frac{K\epsilon^2}{4bc^*}$ vanishes. For a sufficiently small target error $\epsilon < 1$, since $b \geq K$, we can securely bound $\frac{K\epsilon^2}{4bc^*} \leq \frac{1}{c^*}$, obtaining the following lower bound for $\alpha$:
    \begin{align*}
        \alpha &\geq \frac{3\Delta^2}{160K^{3/2}\Rmax^2\left(\sqrt{2}\Rmax + \frac{1}{c^*}\right)}
    \end{align*}
    
    Substituting this lower bound back into the iteration complexity in line~\eqref{eq:iteration_complexity}, a sufficient condition for the number of iterations is:
    \begin{align*}
        T &\geq \frac{640 K^{3/2} \Rmax^2 \left(\sqrt{2}\Rmax + \frac{1}{c^*}\right) c^*}{3 \Delta^2 \epsilon} \log \frac{2\delta_0}{\epsilon}\\
          &= \frac{640 K^{3/2} \Rmax^2 (\sqrt{2}\Rmax c^* + 1)}{3 \Delta^2 \epsilon} \log \frac{2\delta_0}{\epsilon}.
    \end{align*}
    
    \paragraph{Case 2.} Let $\alpha \coloneqq \frac{c^*}{\delta_0}$. Since $c^* \geq 1$, this choice of learning rate is for sure greater than  $\alpha = \frac{\Delta^2}{30K^{3/2}\Rmax^2(\sqrt{2}\Rmax + \frac{2}{\eta}K)}$.
    
    \paragraph{Case 3.} Let $ \alpha = \sqrt{\frac{\eta c^*}{2b}}$. Also, in this case, $\eta c^* >1$ and, since $\eta$ is chosen to be big in order to contrast the bias, this choice of learning rate is for sure greater than the first requirement.
\end{proof}

\subsection{Proofs for Section~\ref{subsec:worst_case}}
Here, we provide the proofs for the Theorems presented in Section~\ref{subsec:worst_case}. First, in Theorem~\ref{thr:convergence} we show that \lbsgb converges to a stationary point of $\barrier$, provided that the learning rate $\alpha$ is properly tuned. Then, we provide a Performance Difference Lemma for the MAB setting in Lemma~\ref{lemma:performance_difference}, along with Lemma~\ref{lemma:approx_global_optimality} showing that approximate first-order stationary points of the regularized objective are approximately globally optimal. Finally, in Theorem~\ref{thr:worst_case} we provide the sample-complexity of \lbsgb without any assumption on the learning objective, and Corollary~\ref{cor:regret} provides a horizon-dependent regret bound for \lbsgb. 

\begin{restatable}[Local Convergence]{theorem}{Convergence}
    \label{thr:convergence}
    Under Assumption \ref{asm:no_ties_r}, and after $T_0 = \mathcal{O}\left(\frac{\eta}{K^{5/2}\alpha^2}\right)$, and choosing the learning rate $\alpha \leq \mathcal{O}\left( \Delta^2/K^{3/2} \right)$, the \lbsgb algorithm guarantees:
    \begin{align*}
        \frac{1}{T}\sum_{t=0}^{T-1}\mE\left[\normGradPhiT^2\right] \leq \mathcal{O}\left( \frac{K^{5/2}\alpha}{\eta\Delta^2}\right).
    \end{align*}
\end{restatable}
\begin{proof}
    Considering the descent lemma in line~\eqref{proof:constant_lr_3}, for any $t \in \dsb{T}$, we have,
    \begin{align}
    \barrierTnext &\geq \barrierT + \left\langle \gradBarrierT, \bt_{t+1} -  \bt_{t}\right\rangle - \left(\frac{5}{3}\normGradPhiT + \frac{15}{\eta}K\right) \| \bt_{t+1} - \bt_{t} \|_2^2\nonumber\\
    & \geq \barrierT + \alpha \left\langle \gradBarrierT, \sampleGradBarrierT\right\rangle - \alpha^2\left(\frac{5}{3}\normGradPhiT + \frac{15}{\eta}K\right)\sampleNormGradPhiT^2\nonumber,
    \end{align}
    where the last inequality follows from the fact that the parameter update is $\bt_{t+1} \leftarrow \bt_t + \alpha \sampleGradBarrierT$.
    
    In the following, we use the notation $\mE_{t}[\cdot]$ to denote the conditional expectation with respect to the history up to the $t$-th time step not included. Formally, consider the filtration defined by the $\sigma$-algebra $\mathcal{F}_t = \sigma(\bt_0, \mathcal{D}_0, \mathcal{D}_1, \dots, \mathcal{D}_t)$ encoding the stochasticity up to time step $t$ included. The stochasticity comes from the samples, excluded the initial parameter $\bt_0$, and the parameter $\bt_{t}$ is deterministically determined by the realization of the samples collected in the first $t-1$ time steps: $\mE_{t}[\cdot] = \mE_{t}[\cdot | \mathcal{F}_{t-1}]$. We will use the fact that $\mE_{t}[X] = X$ for $\mathcal{F}_{t-1}$-measurable $X$. Now, for any $t \in \dsb{T}$:
    \begin{align}
    \barrierTnext 
    &\geq \barrierT + \alpha \mE_t \left[\left\langle \gradBarrierT, \sampleGradBarrierT]\right\rangle\right] - \alpha^2\mE_t\left[\left(\frac{5}{3}\normGradPhiT + \frac{15}{\eta}K\right)\sampleNormGradPhiT^2\right] \nonumber \\
    &\geq \barrierT  + \alpha \mE_t\left[\normGradPhiT\right]^2 \nonumber \\
    &\quad - \alpha^2\left(\frac{5}{3}\normGradPhiT + \frac{15}{\eta}K\right) \left[ \frac{16 \Rmax^3 K^{3/2}}{\Delta^2} \normGradPhiT + \frac{2}{\eta}K\left( \frac{4}{\eta}K + \frac{16 \Rmax^3 K^{3/2}}{\Delta^2} \right)\right]\nonumber\\
    &\geq \barrierT  + \alpha \normGradPhiT^2 - \alpha^2\left(\frac{5}{3}\normGradPhiT + \frac{15}{\eta}K\right) \left[ \frac{16 \Rmax^3 K^{3/2}}{\Delta^2} \normGradPhiT + \frac{S}{\eta} \right]\nonumber\\
    & = \barrierT  + \alpha\left( 1 - \frac{80\Rmax^3 K^{3/2}}{3\Delta^2} \alpha\right)\normGradPhiT^2\nonumber\\
    &\quad- \left[ \left(\frac{5}{3}\frac{S}{\eta} + \frac{240\Rmax^3 K^{5/2}}{\eta\Delta^2}\right)\normGradPhiT + \frac{15}{\eta^2}KS\right]\alpha^2\nonumber\\
    & \geq \barrierT + \alpha\left( 1 - \frac{80\Rmax^3 K^{3/2}}{3\Delta^2} \alpha\right)\normGradPhiT^2\nonumber\\
    &\quad - \left[ \left(\frac{5}{3}\frac{S}{\eta} + \frac{240\Rmax^3 K^{5/2}}{\eta\Delta^2}\right)\left(2\Rmax + \frac{2}{\eta}K\right) + \frac{15}{\eta^2}KS\right]\alpha^2\nonumber\\
    &= \barrierT + \alpha\left( 1 - \frac{80\Rmax^3 K^{3/2}}{3\Delta^2} \alpha\right)\normGradPhiT^2 - \frac{W}{\eta}\alpha^2\nonumber
    \end{align}
    where in the second inequality Lemma~\ref{lemma:strong_growth} has been exploited, $S \coloneqq 2K\left( \frac{4}{\eta}K + \frac{16 \Rmax^3 K^{3/2}}{\Delta^2} \right)$, the last inequality follows from Lemma~\ref{lemma:UB_gradient_norm}, and finally $W \coloneqq \left(\frac{5}{3}S + \frac{240\Rmax^3 K^{5/2}}{\Delta^2}\right)\left(2\Rmax + \frac{2}{\eta}K\right) + \frac{15KS}{\eta}$. Now, applying the law total expectation and telescopic sum for all $t \leq T$ and letting $L \coloneqq \frac{80\Rmax^3 K^{3/2}}{3\Delta^2}$, we have,
    \begin{align}
    \Phi_{\eta}(\bt^*) - \Phi_{\eta}(\bt_0) & \geq \Phi_{\eta}(\bt_T) - \Phi_{\eta}(\bt_0)\nonumber\\
    & = \sum_{t=0}^{T}\mE\left[ \barrierTnext - \barrierT \right]\nonumber\\
    & \geq \sum_{t=0}^{T} \alpha\left( 1 - L \alpha\right)\mE\left[\normGradPhiT^2\right] - \frac{WT}{\eta}\alpha^2\label{proof:ascent_lemma_phi},
    \end{align}
    where $\bt^*\in\arg\max_{\bt\in\mathbb{R}^K}\Phi_\eta(\bt)$ (cf. Appendix \ref{apx:maxima}).
    Now, isolating the sum of the squared gradient norms and dividing everything by $T$, we have,
    \begin{align*}
    \frac{1}{T}\sum_{t=0}^{T} \mE\left[\normGradPhiT^2\right] \leq \frac{\frac{\Phi_{\eta}(\bt^*) - \Phi_{\eta}(\bt_0)}{T} + \frac{W}{\eta}\alpha^2}{\alpha\left( 1 - L \alpha\right)},
    \end{align*}
    which requires the learning rate to be $\alpha < \frac{1}{L}= \frac{3\Delta^2}{80\Rmax^3 K^{3/2}}$ in order to converge. Now, we want to prove that, after some warm-up time, it holds that:
    \begin{align*}
    \frac{1}{T}\sum_{t=0}^{T} \mE\left[\normGradPhiT^2\right] \leq \frac{2W\alpha}{\eta}.
    \end{align*}
    Noting that, when $\alpha \leq 1/2L$, we have $\alpha(1-L\alpha)
    \ge \alpha/2$, we want to find $T_0$ such that,
    \begin{align*}
    \frac{2\left(\Phi_{\eta}(\bt^*) - \Phi_{\eta}(\bt_0)\right)}{\alpha T} < \frac{W\alpha}{\eta},
    \end{align*}
    which holds when:
    \begin{align*}
    T_0 > \mathcal{O}\left(\frac{\Psi\eta}{W\alpha^2}\right),
    \end{align*}
    where $\Psi \coloneqq \Phi_{\eta}(\bt^*) - \Phi_{\eta}(\bt_0)\le 2R_{\max} + \mathcal{B}_\eta(\bt^*)-\mathcal{B}_\eta(\bt_0)\le 2R_{\max}$, since the initial policy is uniform, hence has the largest barrier.
\end{proof}
    
\begin{restatable}[Performance Difference Lemma]{lemma}{perfDiff}
    \label{lemma:performance_difference}
    Let $J(\bt) = \pt^{\top} \bm{r}$ and $A^{\pt}(a) = r(a) - \pt^{\top}\bm{r}$. Given any policy $\pi_{\bt'}$ with performance $J(\bt')$, we have:
    \begin{align*}
    J(\bt') - J(\bt) = \mE_{a \sim \pi_{\bt'}}\left[A^{\pt}(a)\right]
    \end{align*}
\end{restatable}
\begin{proof}
    We have,
    \begin{align*}
    \mE_{a \sim \pi_{\bt'}}\left[A^{\pt}(a)\right] &= \sum_{a \in \dsb{K}}\pi_{\bt'}(a)A^{\pt}(a) \\
    & =  \sum_{a \in \dsb{K}} \pi_{\bt'}(a) \left(r(a) - \pt^{\top}\bm{r}\right) \\
    & = \sum_{a \in \dsb{K}} \pi_{\bt'}(a) r(a) - \sum_{a \in \dsb{K}} \pi_{\bt'}(a) \pt^{\top}\bm{r} \\
    & = \pi_{\bt'}^{\top} \bm{r} - \pt^{\top}\bm{r}\sum_{a \in \dsb{K}}\pi_{\bt'}(a) \\
    & = J(\bt') - J(\bt)
    \end{align*}
\end{proof}

In the next Lemma, we follow~\citet{agarwal2021theory} proving that approximate first-order stationary points of the regularized objective are approximately globally optimal, provided that the barrier parameter $\eta$ is sufficiently big.

\begin{restatable}[Approximate Global Optimality]{lemma}{approxOpt}
    \label{lemma:approx_global_optimality}
    Given a parameter $\bt$ such that:
    \begin{align*}
    \normGradPhi \leq \epsilon_{\mathrm{opt}},
    \end{align*}
    and $\epsilon_{\mathrm{opt}} \le 1/\eta$, we have,
    \begin{align}
    J^* - J(\bt) \leq \frac{K}{\eta}.
    \end{align}
\end{restatable}
\begin{proof}
    Exploiting the Performance Difference Lemma in Lemma~\ref{lemma:performance_difference} and assuming $A^{\pt}(a) \le \frac{2K}{\eta}$ for all $a\in \dsb{K}$, we have,
    \begin{align*}
    J^* - J(\bt) & = \mE_{a \sim \pi^*}\left[A^{\pt}(a)\right] \\
    & = \sum_{a \in \dsb{K}} \pi^*(a) A^{\pt}(a) \\
    & \leq \sum_{a \in \dsb{K}} \pi^*(a)\frac{K}{\eta}\\
    & = \frac{K}{\eta}.
    \end{align*}
    Now, we need to show that, under our assumption that $\normGradPhi \leq \epsilon_{\mathrm{opt}}$, $A^{\pt}(a) \le \frac{2K}{\eta}$ holds for all $a\in\dsb{K}$.\footnote{Under Assumption \ref{asm:no_ties_r}, it would suffice to show this for the unique optimal action, but we keep the proof general.} To do this, we need to bound $A^{\pt}(a)$ for any $a \in \dsb{K}$ where $A^{\pt}(a)$ is positive, else the claim would be trivially true. Let $a$ denote any of these positive-advantage actions in the following. Considering the gradient of our regularized objective function, we have,
    \begin{align}
    \frac{\partial \Phi_{\eta}(\theta_a)}{\partial \theta_a} = \pt(a)A^{\pt}(a) + \frac{1}{\eta} \left(1 - K \pt(a)\right)\label{eq:partial_derivative}.
    \end{align}
    Under our gradient assumption, we have,
    \begin{align*}
    \epsilon_{\mathrm{opt}} \ge \normGradPhi \ge  \left|\frac{\partial \Phi_{\eta}(\theta_a)}{\partial \theta_a}\right|.
    \end{align*}
    Since $\pi_{\bt}(a)>0$ (always true for the softmax parametrization),
    we can solve for $A^{\pt}(a)$ in line~\eqref{eq:partial_derivative},
    \begin{align*}
    A^{\pt}(a) & = \frac{1}{\pt(a)}\left(\frac{\partial \Phi_{\eta}(\theta_a)}{\partial \theta_a} - \frac{1}{\eta} \left(1 - K \pt(a)\right) \right) \\
    &\le \frac{1}{\pt(a)}\left(\left|\frac{\partial \Phi_{\eta}(\theta_a)}{\partial \theta_a}\right| - \frac{1}{\eta} \right) + \frac{K}{\eta}\\
    &\le \frac{1}{\pt(a)}\left(\epsilon_\mathrm{opt} - \frac{1}{\eta} \right) + \frac{K}{\eta}\\
    & \le \frac{K}{\eta},
    \end{align*}
    where we have used our assumption $\epsilon_{\mathrm{opt}} \leq 1/\eta$.
\end{proof}

We can now derive the worst-case sample complexity for \lbsgb leveraging the techniques provided in Corollary 4.11 by~\citet{yuan2022general}. The proof sketch is as follows: we count the number of iterations in which the assumption of Lemma~\ref{lemma:approx_global_optimality} is violated, showing that this number of iterations is bounded by the sum of the squared gradient's norm from Theorem~\ref{thr:convergence}.

\worstCase*
\begin{proof}
    We start defining the set of ``bad'' iterations following~\citet{zhang2021sample}:
    \begin{align*}
    I_{\text{bad}} \coloneqq \left\{ t \in \dsb{T}~~\big|~~ \|\nabla_{\bt}\Phi_\eta(\bt_t)\| > \frac{1}{\eta} \right\}.
    \end{align*}
    By Lemma \ref{lemma:approx_global_optimality}, 
    \begin{align*}
    J^* - \frac{1}{T}\sum_{t=0}^{T-1}J(\bt_t) & = \frac{1}{T} \sum_{t \in I_{\text{bad}}} \left(J^* - J(\theta_t)\right) +  \sum_{t \notin I_{\text{bad}}}\left(J^* - J(\theta_t)\right) \\
    & \leq \frac{|I_{\text{bad}}|}{T} 2\Rmax + \frac{1}{T}\sum_{t \notin I_{\text{bad}}}\left(J^* - J(\theta_t)\right) \\
    & \leq \frac{|I_{\text{bad}}|}{T} 2\Rmax + \frac{T - |I_{\text{bad}}|}{T} \frac{K}{\eta} \\
    & \leq \frac{|I_{\text{bad}}|}{T} 2\Rmax + \frac{K}{\eta}.
    \end{align*}
    Now, we need to find an upper bound for $|I_{\text{bad}}|$. We have,
    \begin{align*}
    \sum_{t=0}^{T-1} \normGradPhiT^2 \geq \sum_{t \in I_{\text{bad}}}\normGradPhiT^2 \geq \frac{|I_{\text{bad}}|}{\eta^2},
    \end{align*}
    from the definition of the set $I_{\text{bad}}$. Thus, we have,
    \begin{align*}
    \frac{|I_{\text{bad}}|}{T} \leq \frac{\eta^2}{T} \sum_{t=0}^{T-1} \normGradPhiT^2.
    \end{align*}
    Now, plugging this upper bound into our previous inequality, we obtain,
    \begin{align*}
    J^* - \frac{1}{T}\sum_{t=0}^{T-1}J(\bt_t) \leq \frac{2\eta^2\Rmax}{T} \sum_{t=0}^{T-1} \normGradPhiT^2 + \frac{K}{\eta}.
    \end{align*}
    Taking the total expectation, we have
    \begin{align}
    J^* - \frac{1}{T}\sum_{t=0}^{T-1}\mE\left[J(\bt_t)\right] \leq \frac{2\eta^2\Rmax}{T} \sum_{t=0}^{T-1} \mE\left[\normGradPhiT^2\right] + \frac{K}{\eta}\label{proof:worst_case_1}.
    \end{align}
    In order to ensure that the average regret is bounded by $\epsilon$, i.e., $J^* - \frac{1}{T}\sum_{t=0}^{T-1}\mE\left[J(\bt_t)\right] \leq \epsilon$, we require both terms on the right-hand side to be bounded by $\epsilon/2$. From the second term, we obtain the requirement for $\eta$:
    \begin{align*}
    \frac{K}{\eta} \leq \frac{\epsilon}{2} \implies \eta \geq 2K\epsilon^{-1}.
    \end{align*}
    To keep the first term as small as possible, we choose the tightest lower bound. Fixing $\eta = 2K\epsilon^{-1}$ in the first term and bounding it by $\epsilon/2$ yields:
    \begin{align*}
    \frac{8 K^2\Rmax}{\epsilon^2}\frac{1}{T}\sum_{t=0}^{T-1} \mE\left[\normGradPhiT^2\right] \leq \frac{\epsilon}{2},
    \end{align*}
    which requires us to let:
    \begin{align*}
    \frac{1}{T}\sum_{t=0}^{T-1} \mE\left[\normGradPhiT^2\right] \leq \frac{\epsilon^3}{16 K^2\Rmax} 
    \end{align*}
    From Theorem~\ref{thr:convergence}, this is equivalent to:
    \begin{align*}
        \frac{W\alpha}{\eta} \leq \frac{K^{-2} \epsilon^{3}}{16\Rmax},
    \end{align*}
    Where $W = \mathcal{O}\left(K^{5/2}\Delta^{-2}\right)$. Since we fixed $\eta = 4K\epsilon^{-1}$ has now been fixed, we obtain the requirement on the learning rate $\alpha$:
    \begin{align*}
    \alpha = \mathcal{O}\left(\Rmax^{-1}\Delta^2 K^{-7/2} \epsilon^{2}\right).
    \end{align*}
    This inequality holds after $T_0$ iterations, where:
    \begin{align*}
    T_0 = \mathcal{O}\left(\frac{\Rmax\Delta^{2}\eta}{K^{5/2}\alpha^2}\right) = \mathcal{O}\left(\Rmax^3\Delta^{-2}K^{11/2}\epsilon^{-5}\right),
    \end{align*}
    which is hence the required sample complexity.
\end{proof}

\begin{restatable}[Regret]{corollary}{Regret}
    \label{cor:regret}
    Under Assumption~\ref{asm:no_ties_r}, for a given horizon $T \geq 1$, selecting a learning rate $\alpha = \mathcal{O}\left(K^{-13/10}\Delta^{6/5}T^{-2/5}\right)$ and the barrier parameter $\eta = \mathcal{O}\left(K^{-1/10}\Delta^{2/5}T^{1/5}\right)$, the \lbsgb ensures sub-linear regret. Specifically, we have,
    \begin{align*}
        \mE\left[\sum_{t=0}^{T-1}\left(\pi^* - \pi_{\bt_t}\right)^{\top}\bm{r}\right] = \mathcal{O}\left(K^{11/10}\Delta^{-2/5}T^{4/5}\right)
    \end{align*}
\end{restatable}
\begin{proof}
    We start by extracting the average sub-optimality bound from the proof of Theorem~\ref{thr:worst_case} prior to fixing $\eta$. Letting $\mE\left[\mathcal{R}(T)\right] \coloneqq T J^* - \sum_{t=0}^{T-1}\mE\left[J(\bt_t)\right]$, we have,
    \begin{align}
      J^* - \frac{1}{T}\sum_{t=0}^{T-1}\mE\left[J(\bt_t)\right] \leq \frac{8\Rmax\eta^2}{T} \sum_{t=0}^{T-1} \mE\left[\normGradPhiT^2\right] + \frac{2K}{\eta}. \label{proof:regret_1}
    \end{align}
    
    Multiplying by $T$, the expected cumulative regret is bounded by,
    \begin{align*}
        \mE\left[\mathcal{R}(T)\right] \leq 8\Rmax\eta^2 \sum_{t=0}^{T-1} \mE\left[\normGradPhiT^2\right] + \frac{2KT}{\eta}.
    \end{align*}
    
    From Theorem~\ref{thr:convergence}, the sum of the expected squared gradient norms is bounded as,
    \begin{align*}
        \sum_{t=0}^{T-1} \mE\left[\normGradPhiT^2\right] = \mathcal{O}\left( \frac{\Rmax}{\alpha} + \frac{K^{5/2}\alpha T}{\eta \Delta^2} \right).
    \end{align*}
    
    Now, plugging this back into the regret bound, we obtain,
    \begin{align}
        \mE\left[\mathcal{R}(T)\right] \leq \mathcal{O}\left( \frac{\Rmax^2 \eta^2}{\alpha} + \frac{\Rmax K^{5/2} \eta \alpha T}{\Delta^2} + \frac{KT}{\eta} \right). \label{proof:regret_2}
    \end{align}
    
    To bound the regret, we jointly optimize the hyperparameters $\alpha$ and $\eta$. Balancing the first two terms with respect to $\alpha$, we need to let:
    \begin{align*}
        \frac{\Rmax^2 \eta^2}{\alpha} = \frac{\Rmax K^{5/2} \eta \alpha T}{\Delta^2} \implies \alpha = \mathcal{O}\left( \sqrt{\frac{\Rmax \Delta^2 \eta}{K^{5/2} T}} \right).
    \end{align*}
    
    Substituting $\alpha$ into line~\eqref{proof:regret_2}, the first two terms collapse, and we obtain,
    \begin{align*}
        \mE\left[\mathcal{R}(T)\right] \leq \mathcal{O}\left( \sqrt{\frac{\Rmax^3 K^{5/2} T}{\Delta^2}} \eta^{3/2} + \frac{KT}{\eta} \right).
    \end{align*}
    
    Next, balancing the remaining terms to find the optimal schedule for $\eta$, we have,
    \begin{align*}
        \sqrt{\frac{\Rmax^3 K^{5/2} T}{\Delta^2}} \eta^{3/2} = \frac{KT}{\eta} \implies \eta^5 = \frac{\Delta^2 T}{\Rmax^3 K^{1/2}},
    \end{align*}
    which leads to:
    \begin{align*}
        \eta = \mathcal{O}\left( K^{-1/10} \Rmax^{-3/5} \Delta^{2/5} T^{1/5} \right).
    \end{align*}
    
    Substituting $\eta$ back into the dominant regret term $\mathcal{O}(KT/\eta)$ yields,
    \begin{align*}
        \mE\left[\mathcal{R}(T)\right] = \mathcal{O}\left( K^{11/10} \Rmax^{3/5} \Delta^{-2/5} T^{4/5} \right)
    \end{align*}
\end{proof}
\section{Additional Results}
\label{subsec:recurrences}

\subsection{Squared Positive Part}
\begin{lemma}[Lower Bound for the Squared Positive Part]
\label{lemma:positive_part_bound}
For any non-negative real numbers $x, y \ge 0$, the following inequality holds:
\begin{align*}
    \left( (x - y)^+ \right)^2 \ge x^2 - 2xy
\end{align*}
\end{lemma}
\begin{proof}
We proceed by analyzing the two cases for the positive part operator.

\paragraph{Case 1.} Let $x \ge y$. The operator is active, hence $(x - y)^+ = x - y$. Expanding the square yields $(x - y)^2 = x^2 - 2xy + y^2$. Since $y^2 \ge 0$, $x^2 - 2xy + y^2 \ge x^2 - 2xy$.

\paragraph{Case 2.} Let $x < y$. The operator is inactive, hence $(x - y)^+ = 0$. We must show that $0 \ge x^2 - 2xy$. Since $x \ge 0$ and $x < y$, multiplying both sides of the strict inequality by $x$ yields $x^2 \le xy$. Because $x, y \ge 0$, we have $xy \le 2xy$. Combining these inequalities yields $x^2 \le 2xy$, which implies $x^2 - 2xy \le 0$.

In both cases, the inequality holds, concluding the proof.
\end{proof}

\subsection{Recurrences}
The goal of this section is to study the recurrence,
\begin{align}
    \label{recurrence}
    r_{t+1} \leq r_t - \alpha \frac{a}{2}r_t^2 + \alpha \frac{b}{\eta},
\end{align}
following the analysis proposed by \citet{montenegro2024learning}. For this purpose, we define the following helper sequence:
\begin{align}
    \begin{cases}
        \rho_0 = r_0\\
        \rho_{t+1} = \rho_t - \alpha\frac{a}{2} \rho_t^2 + \alpha\frac{b}{\eta},\quad\text{if }t\geq0
    \end{cases}
\end{align}
It is possible to show that, under a condition on the step size $\alpha$, the sequence $\rho_t$ upper bounds the recurrence $r_t$.

\begin{lemma}
    If $\alpha \leq \frac{1}{a\rho_t}$, for every $t\geq 0$ we have $\rho_t\geq r_t$.
\end{lemma}
\begin{proof}
    By induction on $t$. For $t=0$, the statement is true since $\rho_0 = r_0$. Suppose the statement holds for $j\leq t$. Then, for $t+1$, we have:
    \begin{align*}
        \rho_{t+1} &= \rho_t - \alpha\frac{a}{2} \rho_t^2 + \alpha\frac{b}{\eta}\\
        & \geq r_t - \alpha\frac{a}{2} r_t^2 + \alpha\frac{b}{\eta}\\
        & \geq r_{t+1}
    \end{align*}
    Where the first inequality holds from the inductive hypothesis and by the fact that the function $f(x) = x - \alpha \frac{a}{2}x^2$ is non-decreasing when $\alpha$ is chosen properly. Indeed, we need to study the sign of the derivative of $f(x)$:
    \begin{align*}
        f'(x) = 1 - \alpha a x \geq 0 \Rightarrow x \leq \frac{1}{\alpha a}
    \end{align*}
    Hence, we need $\alpha \leq \frac{1}{a \rho_t}$ to ensure $\rho_t$ is non-decreasing, and so is $r_t$ by the inductive hypothesis.
\end{proof}
Thus, it is possible to study the convergence of $\rho_t$ as a surrogate for $r_t$. If $\rho_t$ is convergent, than it converges to a fixed point $\bar{\rho}$ as follows:
\begin{align}
    \rhobar = \rhobar  -\alpha \frac{a}{2}\rhobar^2 + \alpha \frac{b}{\eta} \Rightarrow \rhobar = \sqrt{\frac{2b}{\eta a}}\label{eq:rho_bar}
\end{align}
In which only the positive solution of the second-order equation is considered, since $r_t\geq 0$ by definition and $\rho_t \geq r_t$. It is now necessary to study the monotonicity of $\rho_t$. The following lemma states that, under a specific choice of the learning rate, initializing $\rho_0 = r_0$ above the fixed point $\rhobar$, the sequence will be non-increasing, remaining in the interval $[\rhobar, r_0]$. Symmetrically, initializing $\rho_0$ below the fixed point, the sequence will be non-decreasing, remaining in the interval $[r_0, \rhobar]$. 

\begin{lemma}
    The following statements hold:
    \begin{itemize}
        \item if $r_0 \geq \rhobar$ and $\alpha \leq \frac{1}{a r_0}$ it holds that $\rhobar \leq \rho_{t+1} \leq \rho_t$.
        \item if $r_0 \leq \rhobar$ and $\alpha \leq \frac{1}{a \rhobar}$ it holds that $\rhobar \geq \rho_{t+1} \geq \rho_t$.
    \end{itemize}
\end{lemma}
\begin{proof}
    Let's start with the first statement, by induction on $t$. The inductive hypothesis is: $\rho_{t+1} \leq \rho_t$ and $\rho_{t+1} \leq \rhobar$. For $t=0$, for the first inequality, we have,
    \begin{align*}
        \rho_1 &= \rho_0-\alpha \frac{a}{2}\rho_0^2 + \alpha \frac{b}{\eta} \leq \rho_0-\alpha \frac{a}{2}\rhobar^2 + \alpha \frac{b}{\eta}=\rho_0,
    \end{align*}
    where the fact that $\rho_0 \geq \rhobar \geq 0$ and the definition of $\rhobar$ has been used. For the second inequality, we have,
    \begin{align*}
         \rho_1 &= \rho_0-\alpha \frac{a}{2}\rho_0^2 + \alpha \frac{b}{\eta} \leq \rhobar-\alpha \frac{a}{2}\rhobar^2 + \alpha \frac{b}{\eta}=\rhobar,
    \end{align*}
    since the function $x - \alpha \frac{a}{2}x^2$ is non-decreasing in x for $x \leq \rho_0$ since $\alpha \leq \frac{1}{a\rho_0}$ and by definition of $\rhobar$. Suppose now the statements hold for $j\leq t$. Under the inductive hypothesis $\rho_t \leq \rho_0$, the choice of the learning rate $\alpha \leq \frac{1}{a \rho_{0}}$ implies $\alpha \leq \frac{1}{a \rho_{t}}$. Thus, for the first inequality, we have,
    \begin{align*}
        \rho_{t+1} &= \rho_t-\alpha \frac{a}{2}\rho_t^2 + \alpha \frac{b}{2} \leq \rho_t-\alpha \frac{a}{2}\rhobar^2 + \alpha \frac{b}{2}=\rho_t,
    \end{align*}
    where the inductive hypothesis and the definition of $\rhobar$ have been used. For the second inequality, we have,
    \begin{align*}
         \rho_{t+1} &= \rho_t-\alpha \frac{a}{2}\rho_t^2 + \alpha \frac{b}{\eta} \leq \rhobar-\alpha \frac{a}{2}\rhobar^2 + \alpha \frac{b}{\eta}=\rhobar,
    \end{align*}
    where the inductive hypothesis and the fact that $x - \alpha \frac{a}{2}x^2$ is non-decreasing in $x$ for $x \leq \rho_t$, since $\alpha \leq \frac{1}{a \rho_t}$.

    For the second statement, we can proceed analogously as for the first one, switching the signs of the inequalities and recalling that, in this case, $\rho_t$ is upper-bounded by $\rhobar$.
\end{proof}

We can now focus on the case $r_0 \geq \rhobar$, since the second case is irrelevant for the convergence. In this case, we can show that $  \rho_t$ converges to $\rhobar$ with a certain rate. To this end, we can study the following auxiliary sequence:
\begin{align}
    \begin{cases}
        \nu_0 = \rho_0\\
        \nu_{t+1} = \left(1 - \alpha\frac{a}{2}\rhobar\right) \nu_t + \alpha\frac{b}{\eta},\quad\text{if }t\geq0,
    \end{cases}
\end{align}
for which we need to prove that $\nu_t$ upper bounds $\rho_t$.

\begin{lemma}
    If $r_0\geq \rhobar$ and $\alpha \leq \frac{1}{a~r_0}$, then it holds that, for $t\geq0$, $\nu_t \geq \rho_t$.
\end{lemma}
\begin{proof}
    By induction on $t$. For $t=0$, we have that $\nu_0 = \rho_0$, so the statement holds. Suppose the statement holds for $j\leq t$. Then, for $t+1$, we have,
    \begin{align*}
        \nu_{t+1} &= \left(1 - \alpha\frac{a}{2}\rhobar\right) \nu_t + \alpha\frac{b}{\eta}\\
        &\geq \left(1 - \alpha\frac{a}{2}\rhobar\right) \rho_t + \alpha\frac{b}{\eta}\\
        & \geq \left(1 - \alpha\frac{a}{2}\rho_t\right) \rho_t + \alpha\frac{b}{\eta}  =\rho_{t+1},
    \end{align*}
    where the first inequality holds by inductive hypothesis, and the second inequality holds since $\rho_t \geq \rhobar$ and by the fact that $1-\alpha\frac{a}{2}\rho_t \geq 0$ whenever $\alpha \leq \frac{2}{a \rho_t}$, which is entailed by the requirement $\alpha \leq \frac{1}{a\rho_0}$, recalling that $\rho_k \geq 0$ since $\rhobar \geq 0$.
\end{proof}

It is now possible to study the convergence rate of the sequence $\nu_t$, which can be obtained by unrolling the recursion:
\begin{align}
    \nu_{t+1} &= \left(1-\alpha \frac{a}{2}\rhobar\right)^{t+1}\rho_0 + \alpha \frac{b}{\eta} \sum_{j=0}^{t}\left(1-\alpha \frac{a}{2}\rhobar\right)^{j}\label{proof:recurrence_0}\\
    & \leq \left(1-\alpha \frac{a}{2}\rhobar\right)^{t+1}\rho_0 + \alpha \frac{b}{\eta} \sum_{j=0}^{\infty}\left(1-\alpha \frac{a}{2}\rhobar\right)^{j}\nonumber\\
    & = \left(1-\alpha \frac{a}{2}\rhobar\right)^{t+1}\rho_0 + \frac{2b}{\eta a\rhobar}\label{proof:recurrence_1}\\
    & = \left(1-\frac{1}{2}\sqrt{\frac{2\alpha^2 ab}{\eta}}\right)^{t+1}\rho_0 + \sqrt{\frac{2b}{\eta a}}\nonumber,
\end{align}
where Equation~\eqref{proof:recurrence_1} follows from choosing the learning rate $\alpha$ such that the series converges, namely, $\alpha \leq \sqrt{\frac{\eta}{2ab}}$
Putting all the conditions on the learning rate together, we have,
\begin{align*}
    \alpha = \min \left\{ \frac{1}{a~r_0}, \sqrt{\frac{\eta}{2ab}} \right\}
\end{align*}

\subsection{Maxima of the Regularized Objective}\label{apx:maxima}
Let $\Phi_\eta(\bt)=J(\bt)+\frac{1}{\eta}\sum_{a\in \dsb{K}}\log\pi_{\bt}(a)$ be our regularized objective and fix an $\eta>0$. Let $f:\Delta_K\to\mathbb{R}^*$ be the extended real-valued function
\begin{align*}
    f(\bm{x}) = \langle \bm{x},\bm{r}\rangle + \frac{1}{\eta}\sum_i\log\bm{x_i},
\end{align*}
where $\Delta_K$ is the $(K-1)$-dimensional simplex and $\mathbb{R}^*=\mathbb{R}\cup \{+\infty,-\infty\}$. Clearly, $\Delta_K$ is compact and $f$ is upper semi-continuous on $\Delta_K$, so $f$ attains its maximum in $\Delta_K$. Since $f=-\infty$ on the boundary, it must attain its maximum in the interior of $\Delta_K$. The softmax parametrization maps $\mathbb{R}^K$ onto the interior of the simplex (it is surjective). Hence, for every $\bm{x}\in\mathrm{int}(\Delta_K)$, there exists a $\bt_{\bm{x}}\in\mathbb{R}^K$ such that $f(\bm{x})=\Phi_\eta(\bt_{\bm{x}})$.\footnote{In fact, there are infinitely many: if $f(\bm{x})=\Phi_\eta(\bt_{\bm{x}})$, then $f(\bm{x})=\Phi_\eta(\bt)$ for all $\bt\in\Theta_{\bm{x}}$, where $\Theta_{\bm{x}}=\{c\bt_{x}, c\in\mathbb{R}\}$.} Conversely, every $\bt\in\mathbb{R}^K$ defines a unique policy $\bm{x}\in\mathrm{int}(\Delta_K)$. So, $\Phi_\eta$ attains its maximum in $\mathbb{R}^K$.

The same argument applies to the spectral regularization of Section \ref{sec:comparison}, where the interior of the positive-semidefinite cone of the Fisher information matrix is shown to correspond the interior of the simplex under the reparametrized softmax.

\section{Experimental Details}
\label{apx:experimental_details}

In this section, we present the details of the experiments provided in Section~\ref{sec:experiments}, with additional results regarding the comparison of \lbsgb with \sgb, \sgb with entropy regularization (\ent), and Natural Policy Gradient (\npg). Both \sgb and \ent have been described in Appendix~\ref{apx:algorithms}. In the experiments, we focus on the convergence dynamics of the policy $\pi_{\bt_t}$ toward the optimal action $a^*$.

\subsection{Setting}
To evaluate the convergence of the algorithms, we exploit a stationary $K$-armed Gaussian bandit environment. Each action $a \in \dsb{K}$ is associated with a fixed expected value $r(a)$ which remains constant through the time horizon $T$. The reward $R_t(a_t)$ observed by the agent at time $t$ upon selecting action $a_t$ is sampled according to,
\begin{align*} 
    R_t(a_t) \mid a_t = a \sim \mathcal{N}\left(r(a), 1\right),
\end{align*}
implying that the noise $\epsilon_t = R_t(a_t)-r(a_t)$ is i.i.d. following a standard distribution $\mathcal{N}(0, 1)$. The vector of true means $\bm{r} \in \mR^K$ is generated uniformly with a support on $[-\Rmax, \Rmax]$, with a specified $\Delta^*$ with the second-best arm. Note that in the experiments, differently from the theory, $\Rmax$ is the maximum value of the true mean reward, while the sampled reward from each arm, as described previously, is a normal distribution centered at $r(a)$.

\subsection{On the Barrier Parameter $\eta$}

The selection of the hyper-parameter $\eta$ is driven by a trade-off identified in our theoretical analysis. First, $\eta$ must be chosen sufficiently large to minimize the bias of the converged solution relative to the optimal policy. Since the true optimal policy is deterministic, lower values of $\eta$ induce excessive stochasticity, preventing the policy from approximating the optimal distribution.

However, $\eta$ cannot be arbitrarily large. As $\eta \to \infty$, the regularization term vanishes relative to the reward signal, and the objective function recovers the standard \sgb formulation. This would reintroduce the issues inherent to standard policy gradients, such as vanishing gradients or premature convergence, that our regularization aims to mitigate. 

\subsection{Sensitivity to the Number of Arms $K$ and Learning Rate $\alpha$} 
The results are presented in Figure~\ref{fig:experiments_k}, illustrating the convergence dynamics of the policy toward the optimal action $a^*$ across varying action space dimensions $K$ and learning rates $\alpha$. We conducted experiments on bandit instances with $K \in \{10, 100, 1000\}$ using two distinct learning rates:$\alpha = 0.01$ and $\alpha = 0.1$. 

Across all configurations, we fixed the sub-optimality gap to $\Delta^* = 0.1$, the maximum reward to $R_{max} = 1$, and the total time horizon to $T = 2.5 \cdot 10^4$. The barrier parameter $\eta$ for \lbsgb and \ent was scaled according to the action space, as summarized in Table~\ref{tab:eta_values}. All results are averaged over $N=100$ independent runs, and the plots display the corresponding $95\%$ confidence intervals computed via $t$-intervals.

\begin{table}[htbp]
    \centering
    \caption{Choice of the barrier parameter $\eta$ for \lbsgb and \ent across different action spaces $K$ and learning rates $\alpha$.}
    \label{tab:eta_values}
    \renewcommand{\arraystretch}{1.5} 
    \setlength{\tabcolsep}{11pt}      
    \begin{tabular}{|c|c|c|}
        \hline
        \rowcolor[gray]{0.9}
        \textbf{$\boldsymbol{K}$} & \textbf{$\boldsymbol{\alpha = 0.01}$} & \textbf{$\boldsymbol{\alpha = 0.1}$} \\
        \hline
        $10$ & \multicolumn{2}{c|}{$1000$} \\
        \hline
        $100$ & \multicolumn{2}{c|}{$2000$} \\
        \hline
        $1000$ & $10000$ & $5000$ \\
        \hline
    \end{tabular}
\end{table}

All 

\paragraph{Scalability with respect to $K$.} Observing the top row of Figure~\ref{fig:experiments_k} where the learning rate is selected as $\alpha = 0.01$, all the tested algorithms  (\sgb, \lbsgb, \ent, and \npg) successfully identify the optimal arm in the low-dimensional setting ($K = 10$). However, the performance of \sgb, \lbsgb, and \npg deteriorates as $K$ grows to $100$ and $1000$.  In contrast, \lbsgb converges to the optimal policy and demonstrates superior scalability across action space dimensions.

\paragraph{Robustness to the learning rate $\alpha$.} The bottom row of  Figure~\ref{fig:experiments_k}, where the learning rate is selected as $\alpha = 0.1$, exposes the vulnerability of standard PG methods to aggressive step sizes. \sgb, \ent, and \npg exhibit severe instability and premature convergence, since large parameter updates push the policy toward the boundary of the probability simplex. Conversely, \lbsgb remains highly robust, since the repulsive force of the log barrier strictly restricts the optimization trajectory to the non-degenerate interior of the simplex.

\begin{figure}[t]
    \centering
    
    \begin{subfigure}[b]{0.32\textwidth}
        \centering
        \includegraphics[width=\textwidth]{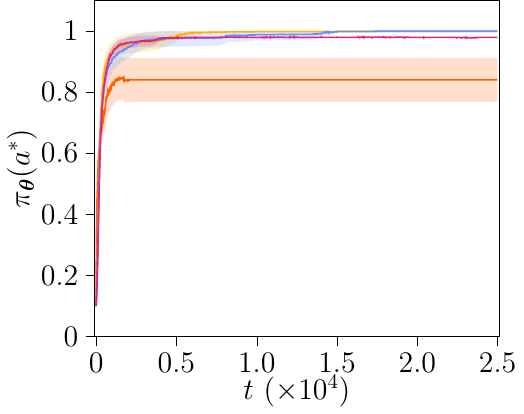}
        \caption{$K=10, \alpha=0.01$}
    \end{subfigure}\hfill
    \begin{subfigure}[b]{0.32\textwidth}
        \centering
        \includegraphics[width=\textwidth]{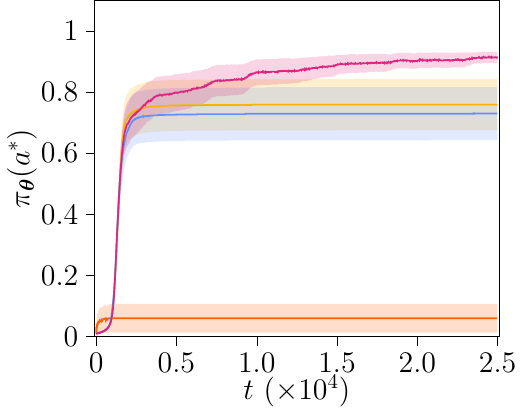}
        \caption{$K=100, \alpha=0.01$}
    \end{subfigure}\hfill
    \begin{subfigure}[b]{0.32\textwidth}
        \centering
        \includegraphics[width=\textwidth]{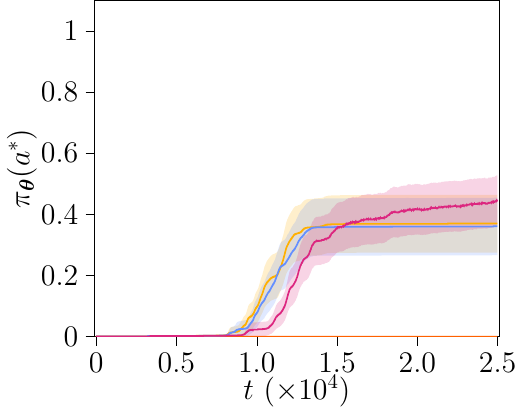}
        \caption{$K=1000, \alpha=0.01$}

    \end{subfigure}
    
    \vspace{1.5em} 
    
    \begin{subfigure}[b]{0.32\textwidth}
        \centering
        \includegraphics[width=\textwidth]{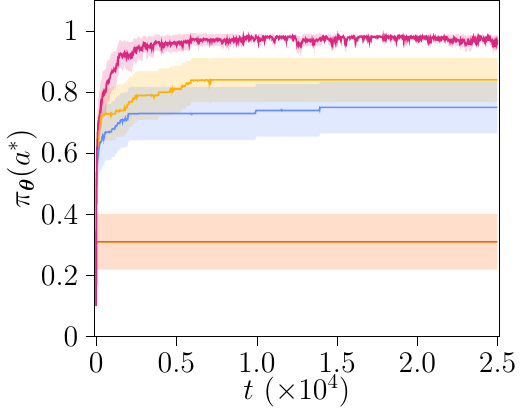}
        \caption{$K=10, \alpha=0.1$}
    \end{subfigure}\hfill
    \begin{subfigure}[b]{0.32\textwidth}
        \centering
        \includegraphics[width=\textwidth]{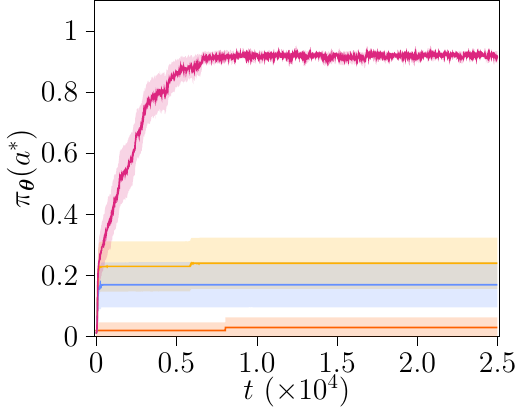}
        \caption{$K=100, \alpha=0.1$}
    \end{subfigure}\hfill
    \begin{subfigure}[b]{0.32\textwidth}
        \centering
        \includegraphics[width=\textwidth]{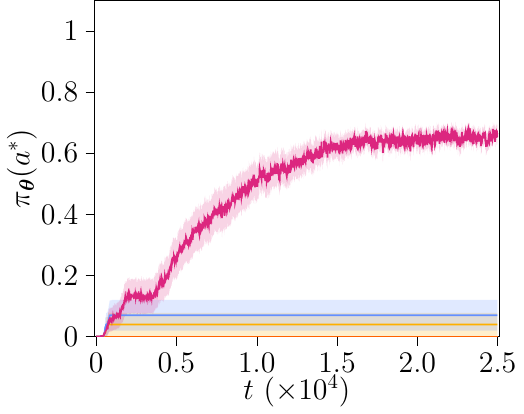}
        \caption{$K=1000, \alpha=0.1$}
    \end{subfigure}
    
    \vspace{1em} 
    
    \centerline{
        \legendline{lbsgb} \lbsgb \hspace{1.5em}
        \legendline{sgb} \sgb \hspace{1.5em}
        \legendline{ent} \ent \hspace{1.5em}
        \legendline{npg} \npg    
    }
    
    \vspace{0.5em} 
    
    \caption{Comparison between algorithms across $K=\{10, 100, 1000\}$. The top row uses a learning rate $\alpha = 0.01$, and the bottom row uses $\alpha = 0.1$. All experiments use $\Delta^* = 0.1$ (100 runs $\pm$ 95\% C.I.). For presentation purposes, these experiments are the same as the ones presented in Figure~\ref{fig:experiments_grid}.}
    \label{fig:experiments_k}
\end{figure}

\subsection{Comparison with Baselines.}
In Figure~\ref{fig:baselines} we compare \lbsgb with baselines. In particular, we compare our proposal with a clipped version of \texttt{LB-SGB} from~\cite{zhang2021sample}, in which the algorithm, at each phase of increasing length, clips the policy to ensure that all actions have a probability of being played greater than $\epsilon_{\text{bb}}\geq 1/2K$. We also compare \lbsgb with \npg with baseline from~\cite{chung2021beyond}, in which a value baseline is subtracted from the importance sampling estimator of the instantaneous reward. The value baseline is estimated by computing the empirical mean of each arm. Both algorithms implement ad-hoc learning rate schedules, while \lbsgb uses a constant learning rate ($\alpha = 0.1$). The experiment shows that the clipping approach induces severe discontinuities in the learning trajectory (visible for $K=10$). Furthermore, its forced exploration threshold ($1/2K$) vanishes for large $K$, failing to prevent premature stagnation. \npg utilizes variance reduction to limit over-committal behavior, but it does not alter the underlying simplex geometry; its aggressive natural gradient updates still extinguish exploration early when initial probabilities are inherently tiny ($1/K$). Conversely, \lbsgb consistently converges to the optimal policy across all instances. By maintaining a continuous log-barrier exploration penalty, \lbsgb demonstrates superior robustness and scalability. Furthermore, it achieves these results with a simple constant learning rate ($\alpha=0.1$), thereby entirely avoiding ad hoc schedules required by the competing methods.

\begin{figure}[t]
    \centering

    \begin{subfigure}[b]{0.32\linewidth}
        \centering
        \includegraphics[width=\linewidth]{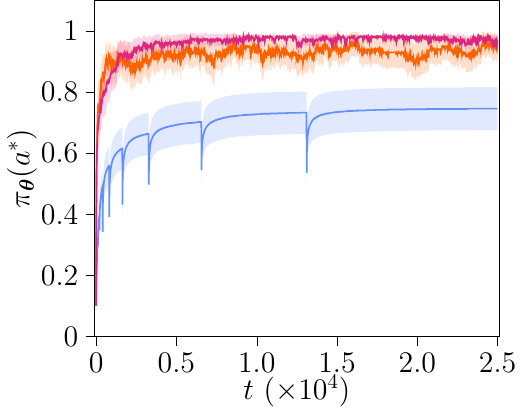}
        \caption{$K=10$}
    \end{subfigure}
    \hfill
    \begin{subfigure}[b]{0.32\linewidth}
        \centering
        \includegraphics[width=\linewidth]{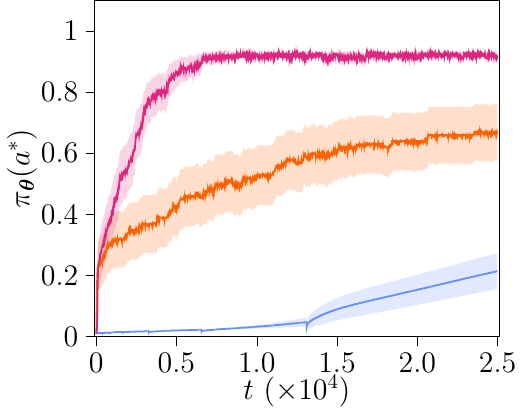}
        \caption{$K=100,$}
    \end{subfigure}
    \hfill
    \begin{subfigure}[b]{0.32\linewidth}
        \centering
        \includegraphics[width=\linewidth]{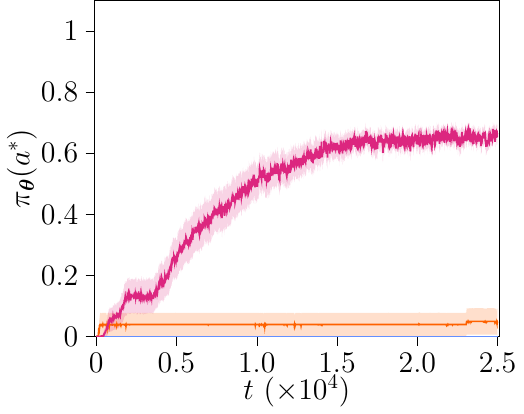}
        \caption{$K=1000$}
    \end{subfigure}

    \vspace{1em} 

    \centerline{
        \legendline{lbsgb} \lbsgb \hspace{1.5em}
        \legendline{npg} \npg  with Baseline \hspace{1.5em}
        \legendline{sgb} \texttt{LB-SGB}  with Clipping
    }

    \vspace{1em}

    \caption{Comparison between \lbsgb, \npg with baseline, and \lbsgb with clipping, across different MAB instances with $K = \{10, 100, 1000\}$. All experiments utilize $\Delta^* = 0.1$, $\Rmax = 1$, and aggregate $100$ independent runs ($\pm 95\%$ C.I.).}
    \label{fig:baselines}
\end{figure}

\subsection{Sensitivity to $\eta$}
In Figure~\ref{fig:eta_sensitivity}, we evaluate the sensitivity of \lbsgb to the barrier parameter $\eta$ on a MAB instance with $K=10$, $\Delta^* = 0.1$, and $R_{\max} = 1$. The three panes illustrate the fundamental trade-off between exploration and exploitation governed by this parameter.

\paragraph{Control over exploration.} The middle pane ($\min_a \pi_{\boldsymbol{\theta}}(a)$) highlights the core failure of vanilla \sgb: lacking any control over the minimum action probability, its exploration rapidly decays to zero. Consequently, the gradient signal vanishes, and the algorithm prematurely stagnates without converging to the optimal action (left pane). In contrast, \lbsgb explicitly bounds this minimum probability away from zero, ensuring continuous exploration throughout the learning process.

\paragraph{The barrier trade-off.} The left pane ($\pi_{\boldsymbol{\theta}}(a^*)$) reveals the behavioral trade-off when tuning $\eta$. If the barrier parameter is too low (e.g., $\eta = 10^2$), the penalty is overly restrictive; the algorithm overexplores and fails to fully concentrate on the optimal arm, plateauing at a highly suboptimal distribution. Conversely, if $\eta$ is too high (e.g., $\eta = 10^4$), the regularization becomes too weak, and the algorithm recovers the unstable, stagnating behavior of vanilla \sgb. When properly tuned (e.g., $\eta = 10^3$), \lbsgb maintains exactly enough exploration to avoid premature convergence while swiftly isolating the optimal action.

\paragraph{Impact on cumulative regret.} This structural trade-off translates directly into the cumulative empirical regret (right pane). The over-regularized instance ($\eta = 10^2$) incurs linear regret due to the persistent, forced sampling of sub-optimal arms. However, we observe a clear optimal choice ($\eta = 10^3$) that yields the lowest cumulative regret, successfully balancing the requirement to explore with the capacity to quickly exploit the optimal action.

\begin{figure}[htbp]
    \centering
    \begin{subfigure}[b]{0.31\textwidth}
        \centering
        \includegraphics[width=\textwidth]{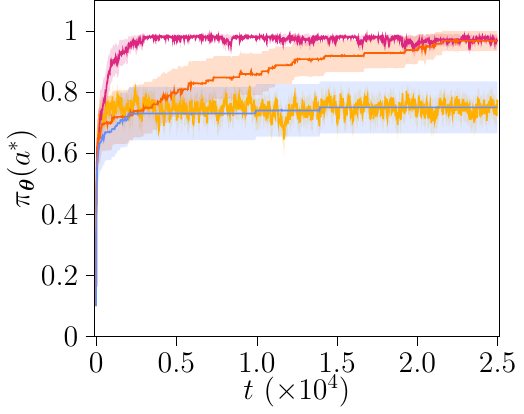}
        \label{fig:sub_k10}
    \end{subfigure}
    \hfill
    \begin{subfigure}[b]{0.32\textwidth}
        \centering
        \includegraphics[width=\textwidth]{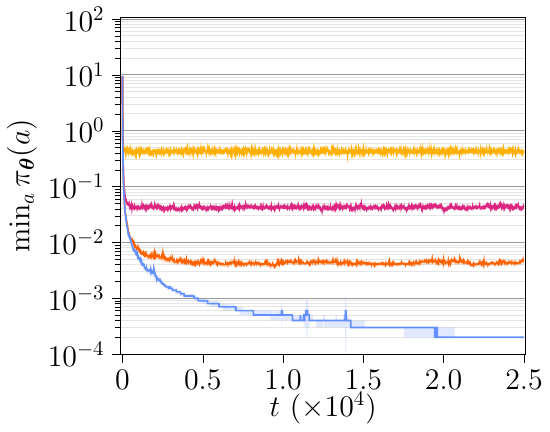}
        \label{fig:sub_k100}
    \end{subfigure}
    \hfill
    \begin{subfigure}[b]{0.31\textwidth}
        \centering
        \includegraphics[width=\textwidth]{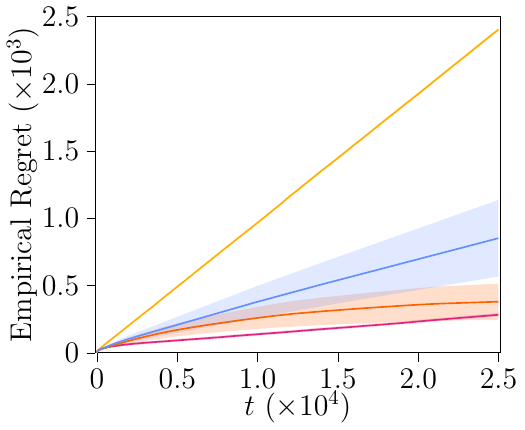}
        \label{fig:sub_k1000}
    \end{subfigure}

    \centerline{
        \legendline{sgb} \sgb \hspace{1.5em}
        \legendline{ent} \texttt{LB-SGB} $\eta = 10^2$ \hspace{1.5em}
        \legendline{lbsgb} \texttt{LB-SGB} $\eta = 10^3$ \hspace{1.5em}
        \legendline{npg} \texttt{LB-SGB} $\eta = 10^4$    
    }

    \vspace{1em}
    
    \caption{Performance, minimum action probability and regret for \sgb and \texttt{LB-SGB} with different choices for $\eta$}
    \label{fig:eta_sensitivity}
\end{figure}
\subsection{Computational Resources}
All the experiments were run on a 2023 14-inch MacBook Pro equipped with an Apple M2 Pro chip and $16~\text{GB}$ of RAM.



\end{document}